\newcommand*{\ICLR}{}
\newcommand*{\CAMREADY}{}

\ifdefined\ICLR
    \documentclass{article}
    \ifdefined\ARXIV
        \usepackage{iclr2018_arxiv,times}
        \iclrfinalcopy
    \else
        \usepackage{iclr2018_conference,times}
        \ifdefined\CAMREADY
		    \iclrfinalcopy
	    \fi
    \fi
\fi

\usepackage{url}
\usepackage{color}
\usepackage{amsfonts}
\usepackage{amsmath}
\usepackage{mathtools}
\usepackage{relsize}
\usepackage{stmaryrd}
\usepackage{graphicx}
\usepackage{nicefrac}
\usepackage{bbm}
\usepackage{wrapfig}
\usepackage[normalem]{ulem}
\usepackage{hyperref}
\usepackage{subcaption}
\usepackage{caption}
\usepackage[titletoc,title]{appendix}
\usepackage{todonotes}
\presetkeys{todonotes}{inline}{}
\ifdefined\COLT
\else
	\usepackage{amsthm}
\fi
\ifdefined\COLT
\else
	
	\newtheorem{lemma}{Lemma}
	\newtheorem{corollary}{Corollary}
	\newtheorem{theorem}{Theorem}
	\newtheorem{proposition}{Proposition}

\fi
\newtheorem{claim}{Claim}
\def\be{\begin{equation}}
\def\ee{\end{equation}}
\def\beas{\begin{eqnarray*}}
\def\eeas{\end{eqnarray*}}
\def\bea{\begin{eqnarray}}
\def\eea{\end{eqnarray}}
\newcommand{\h}{{\mathbf h}}
\newcommand{\x}{{\mathbf x}}
\newcommand{\y}{{\mathbf y}}

\newcommand{\vv}{{\mathbf v}}
\newcommand{\w}{{\mathbf w}}

\newcommand{\aaa}{{\mathbf a}}
\newcommand{\bb}{{\mathbf b}}

\newcommand{\0}{{\mathbf 0}}
\newcommand{\A}{{\mathcal A}}
\newcommand{\B}{{\mathcal B}}

\newcommand{\OO}{{\mathcal O}}

\newcommand{\R}{{\mathbb R}}
\newcommand{\N}{{\mathbb N}}

\newcommand{\abs}[1]{\left\lvert#1 \right\rvert}

\DeclareMathOperator*{\argmax}{argmax} 
\DeclareMathOperator*{\argmin}{argmin}  

\newcommand{\mat}[1]{\llbracket#1\rrbracket}
\newcommand{\rank}[1]{\mathrm{rank}\left(#1\right)}

\newcommand{\fracceil}[2] {\left\lceil\frac{#1}{#2}\right\rceil}

\newcommand{\nicefracfloor}[2] {\left\lfloor\nicefrac{#1}{#2}\right\rfloor}
\newcommand{\cupdot}{\mathbin{\mathaccent\cdot\cup}}
\newcommand{\totstr}{T_S}
\newcommand{\totrec}{T_R}
\newcommand{\tand}{\textrm{ and }}
\ifdefined\CAMREADY
        \newcommand{\githuburl}[1]{\url{https://github.com/HUJI-Deep/#1}}
\else
        \newcommand{\githuburl}[1]{\url{https://<anonymized>}}
\fi

\begin{document}




\ifdefined\ICLR
	\title{On the Expressive Power of Overlapping \\Architectures of Deep Learning}
	\author{Or Sharir \& Amnon Shashua\\
	The Hebrew University of Jerusalem \\
	\texttt{\{or.sharir,shashua\}@cs.huji.ac.il}
	}
	\maketitle
\fi

\begin{abstract}
Expressive efficiency refers to the relation between two architectures A and B,
whereby any function realized by B could be replicated by A, but there exists
functions realized by A, which cannot be replicated by B unless its size grows
significantly larger. For example, it is known that deep networks are
exponentially efficient with respect to shallow networks, in the sense that a
shallow network must grow exponentially large in order to approximate the
functions represented by a deep network of polynomial size. In this work, we
extend the study of expressive efficiency to the attribute of network
connectivity and in particular to the effect of "overlaps" in the convolutional
process, i.e., when the stride of the convolution is smaller than its filter
size (receptive field).
To theoretically analyze this aspect of network's design, we focus on a
well-established surrogate for ConvNets called \emph{Convolutional Arithmetic
Circuits}~(ConvACs), and then demonstrate empirically that our results hold for
standard ConvNets as well. Specifically, our analysis shows that having
overlapping local receptive fields, and more broadly denser connectivity,
results in an exponential increase in the expressive capacity of neural
networks. Moreover, while denser connectivity can increase the expressive
capacity, we show that the most common types of modern architectures already
exhibit exponential increase in expressivity, without relying on fully-connected
layers.

\end{abstract}

\ifdefined\COLT
	\medskip
\fi

\section{Introduction} \label{sec:intro}

One of the most fundamental attributes of deep networks, and the reason for
driving its empirical success, is the ``Depth Efficiency" result which states
that deeper models are exponentially more expressive than shallower models of
similar size. Formal studies of Depth Efficiency include the early work on
boolean or thresholded circuits~\citep{Sipser83,Yao89,Hastad91,Hajnal:1993kl},
and the more recent studies covering the types of networks used in practice
\citep{Pascanu:2013ue,Montufar:2014tb,Eldan:2015uc,expressive_power,
generalized_decomp,Telgarsky:2016wk,Safran:2016te,Raghu:2016wn,BenPoole:2016vy}.
What makes the Depth Efficiency attribute so desirable, is that it brings
exponential increase in expressive power through merely a polynomial change in
the model, i.e. the addition of more layers. Nevertheless, depth is merely one
among many architectural attributes that define modern networks. The deep
networks used in practice consist of architectural features defined by various
schemes of connectivity, convolution filter defined by size and stride, pooling
geometry and activation functions. Whether or not those relate to expressive
efficiency, as depth has proven to be, remains an open question.

In order to study the effect of network design on expressive efficiency we
should first define "efficiency" in broader terms. Given two network
architectures~$A$ and~$B$, we say that architecture~$A$ is expressively
efficient with respect to architecture~$B$, if the following two conditions
hold: \emph{(i)}~any function $\h$ realized by~$B$ of size~$r_B$ can be realized
(or approximated) by $A$ with size $r_A \in\OO(r_B)$; \emph{(ii)}~there exist a
function $\h$ realized by~$A$ with size $r_A$, that cannot be realized (or
approximated) by~$B$, unless  $r_B \in\Omega(f(r_A))$ for some super-linear
function~$f$. The exact definition of the sizes $r_A$ and $r_B$ depends on the
measurement we care about, e.g. the number of parameters, or the number of
``neurons''. The nature of the function~$f$ in condition~\emph{(ii)} determines
the type of efficiency taking place~--~if~$f$ is exponential then
architecture~$A$ is said to be exponentially efficient with respect to
architecture~$B$, and if~$f$ is polynomial so is the expressive efficiency.
Additionally, we say $A$ is \emph{completely efficient} with respect to $B$, if
condition (ii) holds not just for some specific functions (realizable by $A$),
but for all functions other than a negligible set.

In this paper we study the efficiency associated with the architectural
attribute of convolutions, namely the size of convolutional filters (receptive
fields) and more importantly its proportion to their stride. We say that a
network architecture is of the \emph{non-overlapping} type when the size of the
local receptive field in each layer is equal to the stride. In that case, the
sets of pixels participating in the computation of each two neurons in the same
layer are completely separated. When the stride is smaller than the receptive
field we say that the network architecture is of the \emph{overlapping} type. In
the latter case, the \emph{overlapping degree} is determined by the \emph{total}
receptive field and stride projected back to the input layer~--~the implication
being that for the overlapping architecture the total receptive field and stride
can grow much faster than with the non-overlapping case.

As several studies have shown, non-overlapping convolutional networks do have
some theoretical merits. Namely, non-overlapping networks are
universal~\citep{expressive_power,generalized_decomp}, i.e. they can approximate
any function given sufficient resources, and in terms of optimization, under
some conditions they actually possess better convergence guaranties than
overlapping networks. Despite the above, there are only few instances of
strictly non-overlapping networks used in practice (e.g. \citet{tmm,
van2016wavenet}), which raises the question of \textbf{why are non-overlapping
architectures so uncommon?} Additionally, when examining the kinds of
architectures typically used in recent years, which employ a mixture of both
overlapping and non-overlapping layers, there is a trend of using ever smaller
receptive fields, as well as non-overlapping layers having an ever increasing
role~\citep{NiN,Springenberg:2014tx,Szegedy:2014tb}. Hence, the most common
networks used practice, though not strictly non-overlapping, are increasingly
approaching the non-overlapping regime, which raises the question of \textbf{why
having just slightly overlapping architectures seems sufficient for most tasks?}

In the following sections, we will shed some light on these questions by 
analyzing the role of overlaps through a surrogate class of convolutional
networks called Convolutional Arithmetic
Circuits~(ConvACs)~\citep{expressive_power}~--~instead of non-linear activations
and average/max pooling layers, they employ linear activations and product
pooling. ConvACs, as a theoretical framework to study ConvNets, have been the
focused of several works, showing, amongst other things, that many of the
results proven on this class are typically transferable to standard ConvNets as
well~\citep{generalized_decomp,inductive_bias}. Though prior works on ConvACs
have only considered non-overlapping architectures, we suggest a natural
extension to the overlapping case that we call Overlapping ConvACs. In our
analysis, which builds on the known relation between ConvACs and tensor
decompositions, we prove that overlapping architectures are in fact completely
and exponentially more efficient than non-overlapping ones, and that their
expressive capacity is directly related to their \emph{overlapping degree}.
Moreover, we prove that having even a limited amount of overlapping is
sufficient for attaining this exponential separation. To further ground our
theoretical results, we demonstrate our findings through experiments with
standard ConvNets on the CIFAR10 image classification dataset.

\section{Overlapping Convolutional Arithmetic Circuits}
\label{sec:overlapping_convac}

In this section, we introduce a class of convolutional networks referred to as
Overlapping Convolutional Arithmetic Circuits, or Overlapping ConvACs for short.
This class shares the same architectural features as standard ConvNets,
including some that have previously been overlooked by similar attempts to model
ConvNets through ConvACs, namely, having any number of layers and unrestricted
receptive fields and strides, which are crucial for studying overlapping
architectures. For simplicity, we will describe this model only for the case of
inputs with two spatial dimensions, e.g. color images, and limiting the
convolutional filters to the shape of a square.

\begin{wrapfigure}{r}{0.5\textwidth} 
\centering
\includegraphics[width=\linewidth]{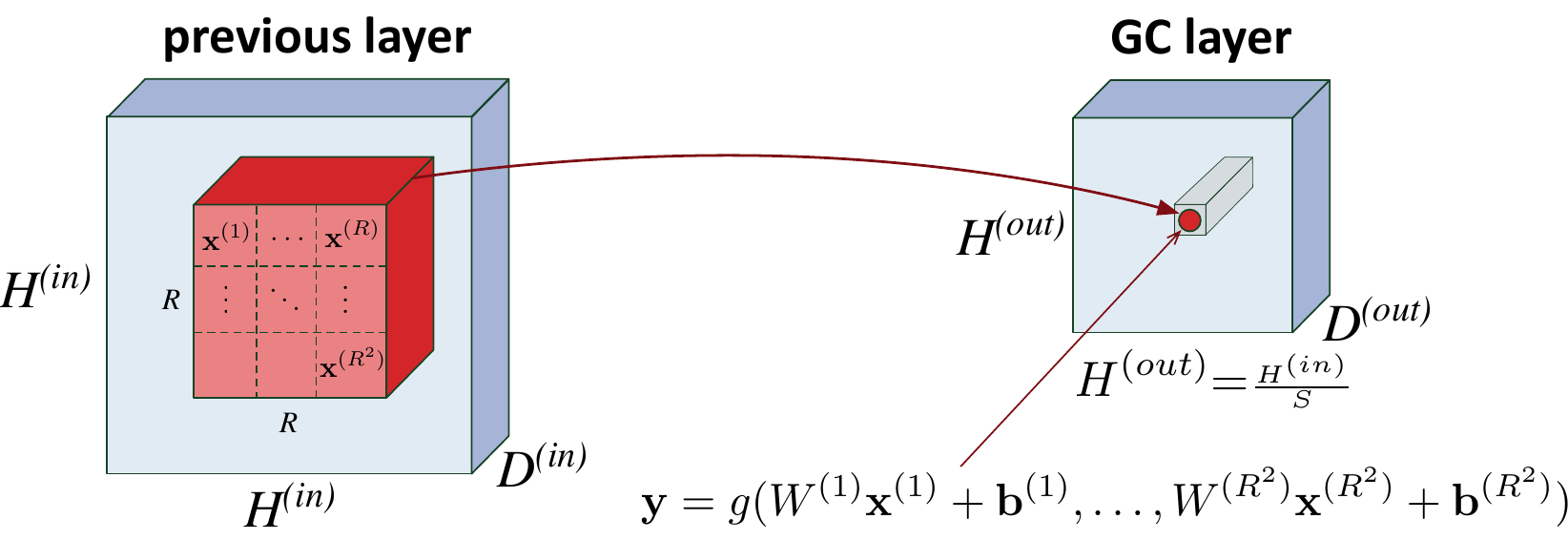}
\caption{An illustration of a GC Layer.}
\label{fig:gc_layer}
\end{wrapfigure}

We begin by presenting a broad definition of a Generalized Convolutional~(GC)
layer as a fusion of a $1{\times}1$ linear operation with a pooling
function~--~this view of convolutional layers is motivated by the
all-convolutional architecture~\citep{Springenberg:2014tx}, which replaces all
pooling layers with convolutions with stride greater than 1. The input to a
GC layer is a 3-order tensor (multi-dimensional array), having width
and height equal to $H^{(\text{in})} \in \N$ and depth $D^{(\text{in})} \in \N$,
also referred to as channels, e.g. the input could be a 2D image with RGB color
channels. Similarly, the output of the layer has width and height equal to
$H^{(\text{out})} \in \N$ and $D^{(\text{out})} \in \N$ channels, where
$H^{(\text{out})} = \frac{H^{(\text{in})}}{S}$ for $S \in \N$ that is referred
to as the \emph{stride}, and has the role of a sub-sampling operation. Each
spatial location $(i,j)$ at the output of the layer corresponds to a 2D window
slice of the input tensor of size $R \times R \times D^{(\text{in})}$, extended
through all the input channels, whose top-left corner is located exactly at
$(i\cdot S, j\cdot S)$, where $R \in \N$ is referred to as its \emph{local
receptive field}, or filter size. For simplicity, the parts of window slices
extending beyond the boundaries have zero value. Let
$\y \in \R^{D^{(\text(out)}}$ be a vector representing the channels at some
location of the output, and similarly, let
$\x^{(1)},\ldots,\x^{(R^2)} \in \R^{D^{(\text{in})}}$ be the set of vectors
representing the slice, where each vector represents the channels at its
respective location inside the $R \times R$ window, then the operation of a
GC layer is defined as follows:
\begin{equation*}
    \y = g(W^{(1)}\x^{(1)}+\bb^{(1)}, \ldots, W^{(R^2)}\x^{(R^2)}+\bb^{(R^2)}),
\end{equation*}
where $W^{(1)},\ldots,W^{(R^2)} \in \R^{D^{(out)} \times D^{(in)}}$
and $\bb^{(1)}, \ldots, \bb^{(R^2)} \in \R^{D^{(out)}}$ are referred to as the
weights and biases of the layer, respectively, and
$g:\R^{D^{(out)}} \times \cdots \times \R^{D^{(out)}} \to \R^{D^{(out)}}$ is
some point-wise pooling function. See fig.~\ref{fig:gc_layer} for an
illustration of the operation a GC layer performs.

\begin{figure}
\centering
\includegraphics[width=0.75\linewidth]{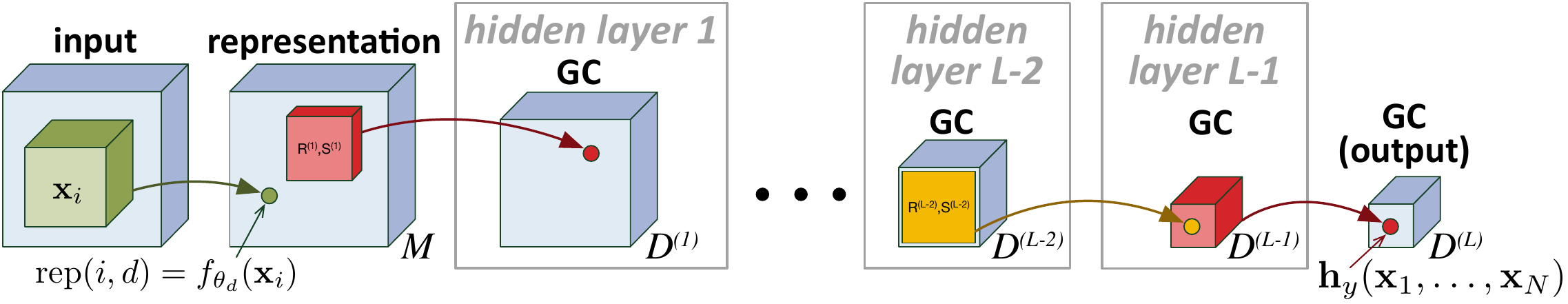}
\caption{An illustration of a Generalized Convolutional Network.}
\label{fig:gc_network}
\end{figure}

With the above definitions, a GC network is simply a sequence of $L$ GC layers,
where for $l \in [L] \equiv \{1,\ldots,L\}$, the $l$'th layer is specified by a
local receptive field $R^{(l)}$, a stride $S^{(i)}$, $D^{(l)}$ output channels,
parameters $\theta^{(l)}$, and a pooling function $g^{(l)}$. For classification
tasks, the output of the last layer of the network typically has $1{\times}1$
spatial dimensions, i.e. a vector, where each output channel
$y \in [Y] \equiv [D^{(L)}]$ represents the score function of the $y$'th class,
denoted by $\h_y$, and inference is perform by $y^* = \arg\max_y h_y(X)$.
Oftentimes, it is common to consider the output of the very first layer of a
network as a low-level feature representation of the input, which is motivated
by the observation that these learned features are typically shared across
different tasks and datasets over the same domain (e.g. edge and Gabor filters
for natural images). Hence, we treat this layer as a separate fixed
``zeroth'' convolutional layer referred to as the \emph{representation} layer,
where the operation of the layer can be depicted as applying a set of fixed
functions $\{f_d:\R^s\to\R\}_{d=1}^M$ to the window slices denoted by
$\x_1,\ldots,\x_N \in \R^s$, i.e. the entries of the output tensor of this layer
are given by $\{f_d(\x_i)\}_{d \in [M], i\in [N]}$. With these notations, the
output of a GC network can be viewed as a function $\h_y(\x_1,\ldots,\x_N)$. The
entire GC network is illustrated in fig.~\ref{fig:gc_network}.

Given a non-linear
point-wise activation function $\sigma(\cdot)$ (e.g. ReLU), then setting all
pooling functions to average pooling followed by the activation, i.e.
$g(\x^{(1)},\ldots,\x^{(R^2)})_c=\sigma\left(\sum_{i=1}^{R^2} x^{(i)}_c\right)$
for $c \in [D^{(\text{out})}]$,
give rise to the common all-convolutional network with $\sigma(\cdot)$
activations, which served as the initial motivation for our formulation.
Alternatively, choosing instead a product pooling function, i.e.
$g(\x^{(1)},\ldots,\x^{(R^2)})_c = \prod_{i=1}^{R^2} x^{(i)}_c$ for
$c \in [D^{(\text{out})}]$, results in an Arithmetic Circuit, i.e.
a circuit containing just product and sum operations, hence it is referred to
as a Convolutional Arithmetic Circuit, or ConvAC. It is important to emphasize
that ConvACs, as originally introduced by \citet{expressive_power}, are
typically described in a very different manner, through the language of tensor
decompositions (see app.~\ref{app:convac} for background). Since vanilla ConvACs
can be seen as an alternating sequence of $1{\times}1$ convolutions and
non-overlapping product pooling layers, then the two formulations coincide when
all GC layers are non-overlapping, i.e. for all $l \in [L]$, $R^{(l)}=S^{(l)}$.
If, however, some of the layers are overlapping, i.e. there exists $l \in [L]$
such that $R^{(l)} > S^{(l)}$, then our formulation through GC layers diverges,
and give rise to what we call \emph{Overlapping ConvACs}.

Given that our model is an extension of the ConvACs framework, it inherits many
of its desirable attributes. First, it shares most of the same traits as modern
ConvNets, i.e. locality, sharing and pooling. Second, it can be shown to form a
universal hypotheses space~\citep{expressive_power}. Third, its underlying
operations lend themselves to mathematical analysis based on measure theory and
tensor analysis~\citep{expressive_power}. Forth, through the concept of
generalized tensor decompositions~\citep{generalized_decomp}, many of the
theoretical results proven on ConvACs could be transferred to standard ConvNets
with ReLU activations. Finally, from an empirical perspective, they tend to work
well in many practical settings, e.g. for optimal classification with missing
data~\citep{tmm}, and for compressed networks~\citep{simnets2}.

While we have just established that the non-overlapping GC Network with a
product pooling function is equivalent to vanilla ConvACs, one might wonder if
using overlapping layers instead could diminish what these overlapping networks
can represent. We show that not only is it not the case, but prove the more
general claim that a network of a given architecture can realize exactly the
same functions as networks using smaller local receptive fields, which includes
the non-overlapping case.
\begin{proposition}\label{prop:nothing_to_lose}
    Let $A$ and $B$ be two GC Networks with a product pooling function. If
    the architecture of $B$ can be derived from $A$ through the removal of
    layers with $1{\times}1$ stride, or by decreasing the local receptive field
    of some of its layers, then for any choice of parameters for $B$, there
    exists a matching set of parameters for $A$, such that the function
    realized by $B$ is exactly equivalent to~$A$. Specifically, $A$ can
    realize any non-overlapping network with the same order of strides
    (excluding $1{\times}1$ strides).
\end{proposition}
\begin{proof}[Proof sketch]
    This follows from two simple claims: (i)~a GC layer can produce an output
    equivalent to that of a GC layer with a smaller local receptive field, by
    ``zeroing'' its weights beyond the smaller local receptive field; and (ii)
    GC layers with $1{\times}1$ receptive fields can be set such that
    their output is equal to their input, i.e. realize the identity function.
    With these claims, the local receptive fields of $A$ can be effectively
    shrank to match the local receptive fields of $B$, and any additional layers
    of $A$ with stride $1{\times}1$ could be set such that they are realizing
    the identity mapping, effectively ``removing'' them from $A$. See
    app.~\ref{app:proofs:nothing_to_lose} for a complete proof.
\end{proof}
Proposition~\ref{prop:nothing_to_lose} essentially means that overlapping
architectures are just as expressive as non-overlapping ones of similar
structure, i.e. same order of non-unit strides. As we recall, this satisfies the
first condition of the efficiency property introduced in sec.~\ref{sec:intro},
and does so regardless if we measure the size of a network as the number of
parameters, or the number of ``neurons''\footnote{We take here the broader
definition of a ``neuron'', as any one of the scalar values comprising the
output array of an arbitrary layer in a network. In the case the output array is
of width and height equal to $H$ and $C$ channels, then the number of such
``neurons'' for that layer is $H^2 \cdot C$.}. In the following section we will
cover the preliminaries required to show that overlapping networks actually lead
to an increase in expressive capacity, which under some settings results in an
exponential gain, proving that the second condition of expressive efficiency
holds as well.

\section{Analyzing Expressive Efficiency Through Grid Tensors}
\label{sec:efficiency_analysis}

In this section we describe our methods for analyzing the expressive efficiency
of overlapping ConvACs that lay the foundation for stating our theorems.
A minimal background on tensor analysis required to follow our work can be found
in sec.~\ref{sec:efficiency_analysis:pre}, followed by presenting our
methods in sec.~\ref{sec:efficiency_analysis:bounds}.

\subsection{Preliminaries}
\label{sec:efficiency_analysis:pre}

In this sub-section we cover the minimal background on tensors analysis required
to understand our analysis. A tensor $\A \in\R^{M_1 \otimes \cdots \otimes M_N}$
of order $N$ and dimension $M_i$ in each mode $i \in [N] \equiv \{1,\ldots,N\}$,
is a multi-dimensional array with entries $\A_{d_1,\ldots,d_N} \in \R$ for all
$i \in [N]$ and $d_i \in [M_i]$. For simplicity, henceforth we assume that all
dimensions are equal, i.e. $M \equiv M_1 = \ldots = M_N$. One of the central
concepts in tensor analysis is that of \emph{tensor matricization}, i.e.
rearranging its entries to the shape of a matrix. Let
$P \cupdot Q=[N]$ be a disjoint partition of its indices, such that
$P = \{p_1,\ldots,p_{|P|}\}$ with $p_1< \ldots< p_{|P|}$,
and $Q = \{q_1, \ldots, q_{|Q|}\}$ with $q_1 < \ldots < q_{|Q|}$. The
matricization of $\A$ with respect to the partition
$P \cupdot Q$, denoted by $\mat{\A}_{P,Q}$, is the $M^{|P|}$-by-$M^{|Q|}$ matrix
holding the entries of $\A$, such that for all $i \in [N]$ and $d_i \in [M]$ the 
entry $A_{d_1, \ldots, d_N}$ is placed in row index
${1 + \sum_{t=1}^{|P|} (d_{p_t} - 1) M^{|P| - t}}$ and column index
${1 + \sum_{t=1}^{|Q|} (d_{q_t} - 1) M^{|Q| - t}}$. Lastly, the tensors we
study in this article originate by examining the values of some given function
at a set of predefined points and arranging them in a tensor referred to as
the \emph{grid tensor} of the function. Formally, let
$f:\R^s \times \ldots \times \R^s \to \R$ be a function, and let
$\{\x^{(1)}, \ldots, \x^{(M)} \in \R^s\}$ be a set of vectors called
\emph{template vectors}, then the grid tensor of $f$ is denoted by
$\A(f) \in \R^{M \otimes \ldots \otimes M}$ and defined by
$\A(f)_{d_1,\ldots,d_N} = f(\x^{(d_1)}, \ldots, \x^{(d_N)})$ for all
$d_1,\ldots,d_N \in [M]$.

\subsection{Bounding the Size of Networks via Grid Tensors}
\label{sec:efficiency_analysis:bounds}

We begin with a discussion on how to have a well-defined measure of efficiency.
We wish to compare the efficiency of non-overlapping ConvACs to overlapping
ConvACs, for a fixed set of $M$ representation functions (see
sec.~\ref{sec:overlapping_convac} for definitions). While all functions
realizable by non-overlapping ConvACs with shared representation functions lay
in the same function subspace (see \citet{expressive_power}), this is not the
case for overlapping ConvACs, which can realize additional functions outside the
sub-space induced by non-overlapping ConvACs. We cannot therefore compare both
architectures directly, and need to compare them through an auxiliary objective.
Following the work of \citet{generalized_decomp}, we instead compare
architectures through the concept of grid tensors, and specifically, the grid
tensor defined by the output of a ConvAC, i.e. the tensor $\A(\h)$ for
$\h(\x_1,\ldots,\x_N)$. Unlike with the ill-defined nature of directly comparing
the functions of realized by ConvACs, \citet{generalized_decomp} proved that
assuming the fixed representation functions are linearly independent, then there
exists template vectors $\x^{(1)},\ldots,\x^{(M)}$, for which any
non-overlapping ConvAC architecture could represent all possible grid tensors
over these templates, given sufficient number of channels at each layer. More
specifically, if $F_{ij} = f_i(\x^{(j)})$, then these template vector are
chosen such that $F$ is non-singular. Thus, once we fix a set of linearly
independent representation functions, we can compare different ConvACs, whether
overlapping or not, on the minimal size required for them to induce the same
grid tensor, while knowing such a finite number always exists.

One straightforward direction for separating between the expressive efficiency
of two network architectures A and B is by examining the ranks of their
respective matricized grid tensors. Specifically, Let $\A(\h^{(A)})$ and
$\A(\h^{(B)})$ denote the grid tensors of A and B, respectively, and let $(P,Q)$
be a partition of $[N]$, then we wish to find an upper-bound on the rank of
$\mat{\A(\h^{(A)})}_{P,Q}$ as a function of its size on one hand, while showing
on the other hand that $\rank{\mat{\A(\h^{(B)})}_{P,Q}}$ can be significantly
greater. One benefit of studying efficiency through a matrix rank is that not
only we attain separation bounds for exact realization, but also immediately
gain access to approximation bounds by examining the singular values of the
matricized grid tensors. This brings us to the following lemma, which connects
upper-bounds that were previously found for non-overlapping
ConvACs~\citep{inductive_bias}, with the grid tensors induced by them (see
app.~\ref{app:proofs:preliminaries} for proof):
\begin{lemma} \label{lemma:mat_rank_bound}
    Let $h_y(\x_1,\ldots,\x_N)$ be a score function of a non-overlapping
    ConvAC with a fixed set of $M$ linearly independent and continuous
    representation functions, and $L$ GC layers. Let $(P,Q)$ be a partition
    dividing the spatial dimensions of the output of the representation layer
    into two equal parts, either along the horizontal or vertical axis, referred
    to as the ``left-right'' and ``top-bottom'' partitions, respectively. Then,
    for any template vectors such that $F$ is non-singular and for any choice of
    the parameters of the network, it holds that
    $\rank{\mat{\A(\h_y)}_{P,Q}} \leq D^{(L-1)}$.
\end{lemma}

Lemma~\ref{lemma:mat_rank_bound} essentially means that it is sufficient to show
that overlapping ConvACs can attain ranks super-polynomial in their size to
prove they are exponentially efficient with respect to non-overlapping ConvACs.
In the next section we analyze how the overlapping degree is related to the
rank, and under what cases it leads to an exponentially large rank.

\section{The Expressive Efficiency of Overlapping Architectures}
\label{sec:overlaps_efficiency}

In this section we analyze the expressive efficiency of overlapping
architectures. We begin by defining our measures
of the overlapping degree that will used in our claims, followed by presenting
our main results in sec.~\ref{sec:main_results}. For the sake of
brevity, an additional set of results, in light of the recent work by
\citet{inductive_bias} on ``Pooling Geometry'', is deferred to
app.~\ref{app:pooling_geometry}.

\subsection{The Overlapping Degree of a Network}
\label{sec:overlapping_degree}

\begin{wrapfigure}{r}{0.5\textwidth} 
\vspace{-3mm}
\centering
\includegraphics[width=\linewidth]{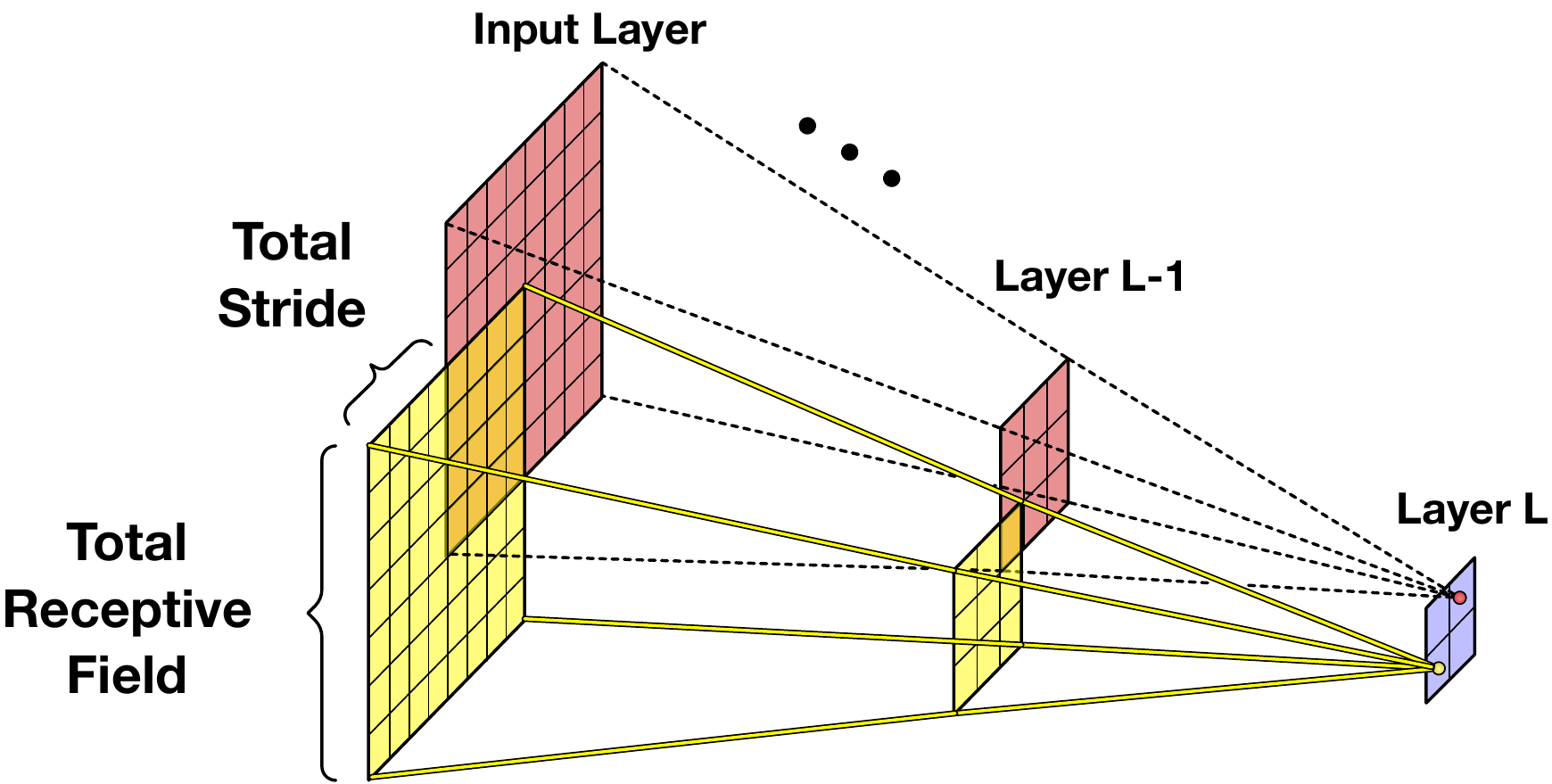}
\caption{Illustrating the total receptive field and total stride
         attributes for the $L$'th layer, which could be seen as the projected
         receptive field and stride with respect to the input layer. Together,
         they capture the overlapping degree of a network.}
\label{fig:total_properties}
\vspace{-3mm}
\end{wrapfigure}

To analyze the efficiency of overlapping architectures, we will first formulate
more rigorously the measurement of the overlapping degree of a given
architecture. As mentioned in sec.~\ref{sec:intro}, we do so by defining the
concepts of the \emph{total receptive field} and \emph{total stride} of a given
layer $l\in[L]$, denoted by $\totrec^{(l)}$ and $\totstr^{(l)}$, respectively.
Both measurements could simply be thought of as projecting the accumulated local
receptive fields (or strides) to the the first layer, as illustrated in
fig.~\ref{fig:total_properties}, which represent a type of global statistics
of the architecture. However, note that proposition~\ref{prop:nothing_to_lose}
entails that a given architecture could have a smaller \emph{effective} total
receptive field, for some settings of its parameters. This leads us to define
the $\alpha$-minimal total receptive field, for any $\alpha\in\R_+$, as the
smallest effective total receptive field still larger than $\alpha$, which we
denote by $\totrec^{(l,\alpha)}$. The exact definitions of the above concepts
are formulated as follows:
\begin{align}
    \label{eq:total_stride}
    \totstr^{(l)} \equiv \totstr^{(l)}(S^{(1)},\ldots,S^{(l)})
        &\equiv \begin{cases}
            \prod_{i=1}^l S^{(i)} & l \geq 1\\
            1 & l = 0
        \end{cases} \\
    \label{eq:total_receptive}
    \totrec^{(l)} \equiv \totrec^{(l)}(R^{(1)},S^{(1)},\ldots,R^{(l)},S^{(l)})
        &\equiv R^{(l)} \cdot \totstr^{(l-1)} +
         \sum\nolimits_{k=1}^{l-1}\left(R^{(k)}-S^{(k)}\right)\cdot \totstr^{(k-1)} \\
    \label{eq:minimal_receptive}
    \totrec^{(l,\alpha)}
        \equiv \totrec^{(l,\alpha)}(R^{(1)},S^{(1)},\ldots,R^{(l)},S^{(l)})
        &\equiv \smashoperator[r]{\argmin_{\substack{
            \forall i\in[l], S^{(i)} \leq t_i \leq R^{(i)} \\
            \totrec^{(l)}(t_1,S^{(1)},\ldots,t_l,S^{(l)}) > \alpha}}}
         \quad\,\,\totrec^{(l)}(t_1,S^{(1)},\ldots,t_l,S^{(l)})
\end{align}
where we omitted the arguments of $\totstr^{(l-1)}$ and $\totstr^{(k-1)}$ for
the sake of visual compactness.

Notice that for non-overlapping networks the total receptive field always
equals the total stride, and that only at the end of the network, after the
spatial dimension collapses to $1{\times}1$, does the the total receptive field
grow to encompass the entire size of the representation layer. For overlapping
networks this is not the case, and the total receptive field could grow much
faster. Intuitively, this means that values in regions of the input layer that
are far apart would be combined by non-overlapping networks only near the last
layers of such networks, and thus non-overlapping networks are effectively
shallow in comparison to overlapping networks. Base on this intuition, in the
next section we analyze networks with respect to the point at which their total
receptive field is large enough.

\subsection{Main Results}\label{sec:main_results}

With all the preliminaries in place, we are ready to present our main result:
\begin{theorem}\label{thm:main_overlaps}
    Assume a ConvAC with a fixed representation layer having $M$ output channels
    and both width and height equal to $H$, followed by $L$ GC layers, where the
    $l$'th layer has a local receptive field $R^{(l)}$, a stride $S^{(l)}$, and
    $D^{(l)}$ output channels. Let $K \in [L]$ be a layer with a total receptive
    field $\totrec^{(K)} \equiv \totrec^{(K)}(R^{(1)},S^{(1)},
    \ldots,R^{(K)},S^{(K)})$, such that $\totrec^{(K)}>\frac{H}{2}$.
    Then, for any choice
    of parameters, except a null set (with respect to the Lebesgue
    measure), and for any template vectors such that $F$ is non-singular, the
    following equality holds:
    \begin{align}\label{eq:lower_bound}
        \rank{\mat{\A(\h_y)}_{P,Q}} \geq D^{\left\lfloor \frac{H - \totrec^{(K, \left\lfloor \nicefrac{H}{2} \right\rfloor)}}{\totstr^{(K)}} +1 \right\rfloor \cdot \left\lceil \frac{H}{\totstr^{(K)}} \right\rceil}
    \end{align}
    where $(P,Q)$ is either the ``left-right'' or the ``top-bottom'' partitions
    and ${D \equiv \min\{M,D^{(K)},\frac{1}{2} \min_{1\leq l\leq K} D^{(l)}\}}$.
\end{theorem}
\begin{proof}[Proof sketch]
    Because the entries of the matricized grid tensors are polynomials in the
    parameters, then according to a lemma by \citet{tmm}, if there is a single
    example that attains the above lower-bound on the rank, then it occurs
    almost everywhere with respect to the Lebesgue measure on the Euclidean
    space of the parameters.

    Given the last remark, the central part of our proof is simply the
    construction of such an example. First we find a set of parameters for the
    simpler case where the first GC layer is greater than a quarter of the
    input, satisfying the conditions of the theorem. The motivation behind the
    specific construction is the pairing of indices from each side of
    the partition, such that they are both in the same local receptive field,
    and designing the filters such that the output of each local application of
    them defines a mostly diagonal matrix of rank $D$, with respect to
    these two indices. The rest of the parameters are chosen such that the
    output of the entire network results in a product of the entries of these
    matrices. Under matricization, this results in a matrix who is
    equivalent\footnote{Two matrices are equivalent if one could be converted
    to the other by elementary row or column operations.} to a
    Kronecker product of mostly diagonal matrices. Thus, the matricization rank
    is equal to the product of the ranks of these matrices, which results in the
    exponential form of eq.~\ref{eq:lower_bound}.
    Finally, we extend the above example to the general case, by realizing the
    operation of the first layer of the above example through multiple layers
    with small local receptive fields. See app.~\ref{app:proofs:preliminaries}
    for the definitions and lemmas we rely on, and see
    app.~\ref{app:proofs:main_overlaps} for a complete proof.
\end{proof}
Combined with Lemma~\ref{lemma:mat_rank_bound}, it results in the following
corollary:
\begin{corollary}
    Under the same setting as theorem~\ref{thm:main_overlaps}, and for all
    choices of parameters of an overlapping ConvAC, except a negligible set,
    any non-overlapping ConvAC that realizes (or approximates) the same grid
    tensor must be of size at least:
    \begin{equation*}
        D^{\left\lfloor \frac{H - \totrec^{(K, \left\lfloor \nicefrac{H}{2} \right\rfloor)}}{\totstr^{(K)}} +1 \right\rfloor \cdot \left\lceil \frac{H}{\totstr^{(K)}} \right\rceil} .
    \end{equation*}
\end{corollary}

While the complexity of the generic lower-bound above might seem
incomprehensible at first, its generality gives us the tools to analyze
practically any kind of feed-forward architecture. As an example, we can analyze
the lower bound for the well known GoogLeNet
architecture~\citep{Szegedy:2014tb}, for which the lower bound equals $32^{98}$,
making it clear that using a non-overlapping architecture for this case is
infeasible. Next, we will focus on specific cases for which we can derive more
intelligible lower bounds.

\begin{figure}
\centering
\includegraphics[width=0.8\linewidth]{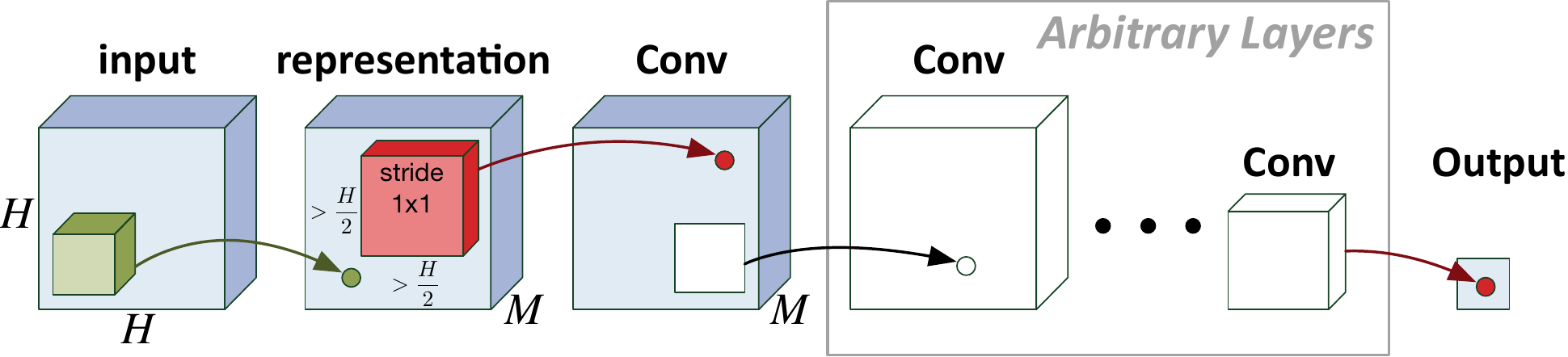}
\caption{A network architectures beginning with large local receptive fields
    greater than $\nicefrac{N}{2}$ and at least $M$ output channels. According
    to theorem~\ref{thm:main_overlaps}, for almost all choice of parameters
    we obtain a function that cannot be approximated by a non-overlapping
    architecture, if the number of channels in its next to last layer is
    less than $M^{\frac{H^2}{2}}$.}
\label{fig:start_with_big_conv}
\end{figure}

According to theorem~\ref{thm:main_overlaps}, the lower bound depends on the
first layer for which its total receptive field is greater than a quarter of the 
input. As mentioned in the previous section, for non-overlapping networks this
only happens after the spatial dimension collapses to $1{\times}1$, which
entails that both the total receptive field and total stride would be equal to
the width $H$ of the representation layer, and substituting this values in
eq.~\ref{eq:lower_bound} results simply in $D$~--~trivially meaning that to
realize one non-overlapping network by another non-overlapping network, the next
to last layer must have at least half the channels of the target network.

On the other extreme, we can examine the case where the first GC layer has a
local receptive field $R$ greater than a quarter of its input, i.e.
$R > \nicefrac{H}{2}$. Since the layers following the first GC layer do not
affect the lower bound in this case, it applies to any arbitrary sequence of
layers as illustrated in fig.~\ref{fig:start_with_big_conv}. For simplicity
we will also assume that the stride $S$ is less than $\nicefrac{H}{2}$, and that
$\frac{H}{2}$ is evenly divided by $S$. In this case the $\frac{H}{2}$-minimal
receptive field equals to $\frac{H}{2} + 1$, and thus the lower bound results in
$D^{\frac{H^2}{2S}}$. Consider the case of $D = M$ and $S=1$, then a
non-overlapping architecture that satisfies this lower bound is of the order of
magnitude at which it could already represent any possible grid tensor. This
demonstrate our point from the introduction, that through a a polynomial change
in the architecture, i.e. increasing the receptive field, we get an exponential
increase in expressivity.

\begin{figure}
\centering
\includegraphics[width=0.9\linewidth]{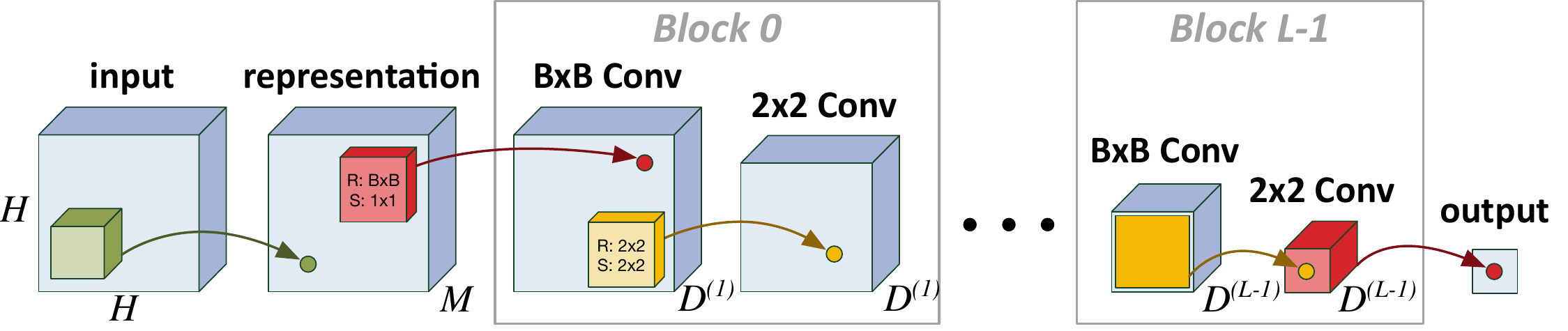}
\caption{The common network architecture of alternating $B{\times}B$ ``conv''
    and $2{\times}2$ ``pooling'' layers. If $B \leq \nicefrac{H}{5}{+}1$ and
    $D^{(l)} \geq 2M$ for all $1 \leq l < L$, then the lower bound of
    theorem~\ref{thm:main_overlaps} for this network results in
    $M^{\frac{(2B - 1)^2}{4}}$.
}
\label{fig:alt_conv_pool_net}
\end{figure}

Though the last example already demonstrates that a polynomially sized
overlapping architecture could lead to an exponential separation, in practice,
employing such large convolutions is very resource intensive. The common best
practice is to use multiple small local receptive fields of size
$B \times B$, where the typical values are $B=3$ or $B=5$, separated by a
$2 \times 2$ ``pooling'' layers, i.e. layers with both stride and local
receptive field equal to $2 \times 2$. For simplicity, we assume that $H = 2^L$
for some $L \in \N$. See fig.~\ref{fig:alt_conv_pool_net} for an illustration of
such a network. Analyzing the above network with theorem~\ref{thm:main_overlaps}
results in the following proposition:
\begin{proposition}\label{prop:common_case}
    Consider a network  comprising a sequence of GC blocks, each block begins
    with a layer whose local receptive field is $B {\times} B$ and its stride
    $1{\times}1$, followed by a layer with local receptive field $2 {\times} 2$
    and stride $2 {\times} 2$, where the output channels of all layers are at
    least $2M$, and the spatial dimension of the representation layer is
    $H {\times} H$ for $H{=}2^L$. Then, the lower bound describe by
    eq.~\ref{eq:lower_bound} for the above network is greater than or equal to:
    \begin{align*}
    \tau(B,H) &\equiv M^{\frac{(2B-1)^2}{2} \cdot \left(1+\frac{2B-2}{H}\right)^{-2}} 
    = M^{\frac{H^2}{2}\cdot \left(1 + \frac{H-1}{2B-1} \right)^{-2}},
    \end{align*}
    whose limits are $\lim_{B\to\infty} \tau(B,H) = M^\frac{H^2}{2}$ and
    $\lim_{H\to\infty} \tau(B,H) = M^\frac{(2B-1)^2}{2}$. Finally, assuming
    $B \leq \frac{H}{5} + 1$, then $\tau(B,H) \geq M^{\frac{(2B-1)^2}{4}}$.
\end{proposition}
\begin{proof}[Proof sketch]
    We first find a closed-form expression for the total receptive field and
    stride of each of the $B{\times}B$ layers in the given network. We
    then show that for layers whose total receptive field is greater than
    $\frac{H}{2}$, its $\alpha$-minimal total receptive field, for
    $\alpha{=}\frac{H}{2}$, is equal to $\frac{H}{2}{+}1$. We then use the above
    to find the first layer who satisfies the conditions of
    theorem~\ref{thm:main_overlaps}, and then use our closed-forms expressions
    to simplify the general lower bound for this case. See
    app.~\ref{app:proofs:common_case} for a complete proof.
\end{proof}
In particular, for the typical values of $M=64$, $B=5$, and $H \geq 20$, the
lower bound is at least $64^{20}$, which demonstrates that even having a small
amount of overlapping already leads to an exponential separation from the
non-overlapping case. When $B$ grows in size, this bound approaches the earlier
result we have shown for large local receptive fields encompassing more than a
quarter of the image. When $H$ grows in size, the lower bound is dominated
strictly by the local receptive fields. Also notice that based on
proposition~\ref{prop:common_case}, we could also derive a respective lower
bound for a network following VGG style architecture~\citep{Simonyan:2014ws},
where instead of a single convolutional layer before every ``pooling'' layer, we
have $K$ layers, each with a local receptive field of $C \times C$. Under this
case, it is trivial to show that the bound from
proposition~\ref{prop:common_case} holds for $B = K \cdot (C-1) + 1$, and under
the typical values of $C=3$ and $K=2$ it once again results in a lower bound of
at least $64^{20}$.

\section{Experiments}\label{sec:exp}

\begin{figure}
\centering
\includegraphics[width=0.49\linewidth]{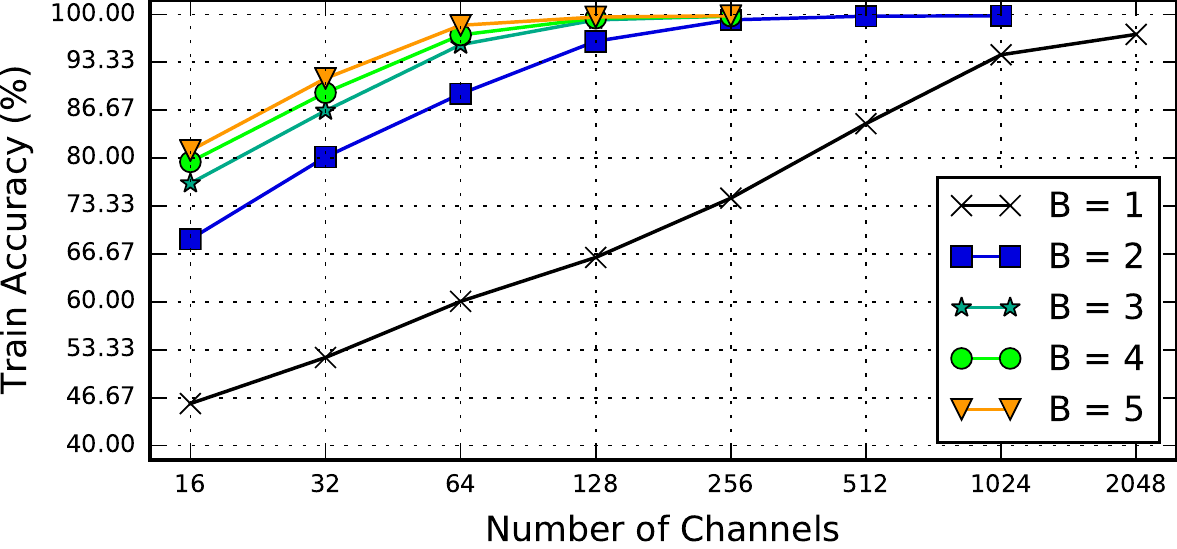}~~~~\includegraphics[width=0.49\linewidth]{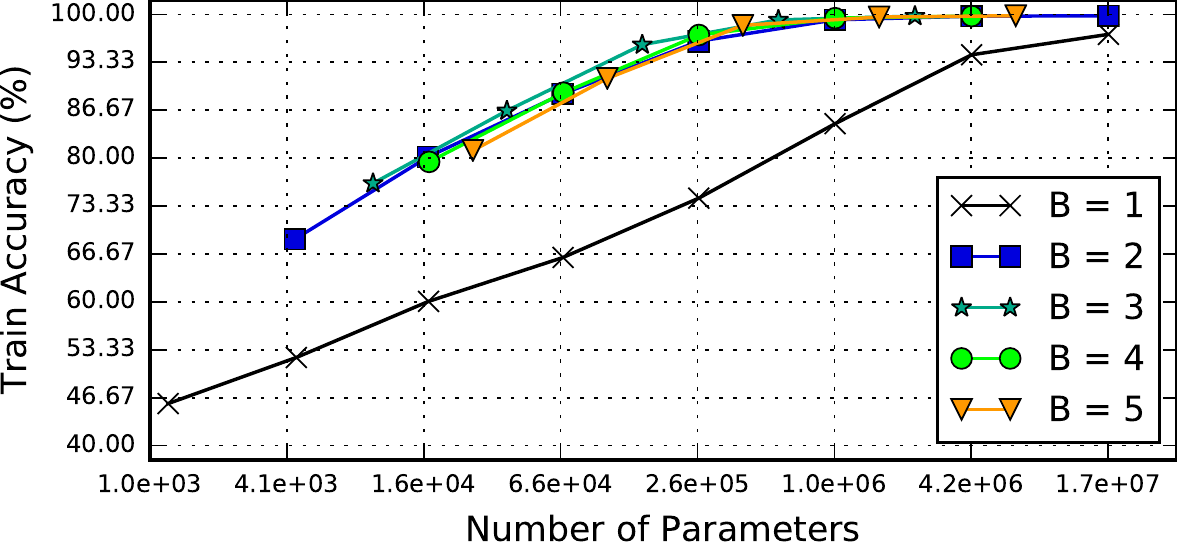}

\includegraphics[width=0.49\linewidth]{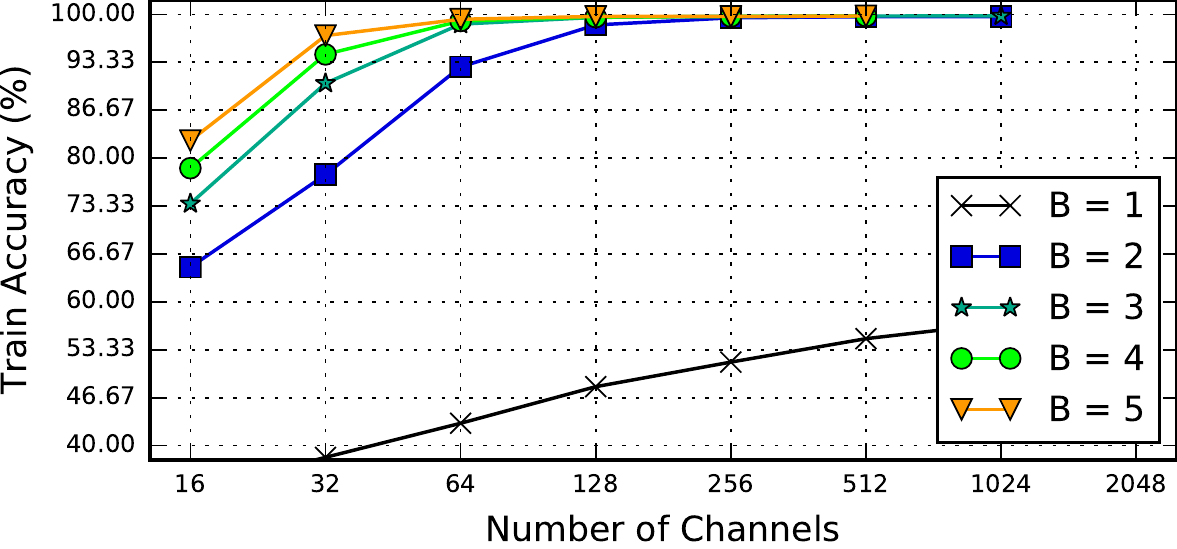}~~~~\includegraphics[width=0.49\linewidth]{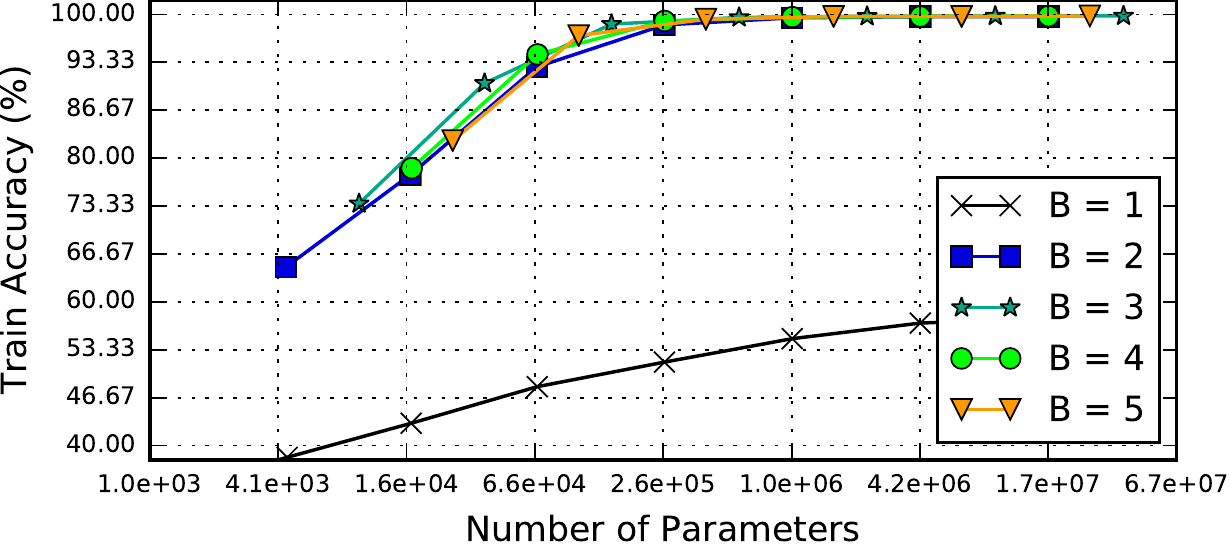}

\caption{Training accuracies of standard ConvNets on CIFAR-10 with data
    augmentations, where the results of spatial augmentations presented at the
    top row, and color augmentations at the bottom row. Each network follows the
    architecture of proposition~\ref{prop:common_case}, with with receptive
    field $B$ and using the same number of channels across all layers, as
    specified by the horizontal axis of left plot. We plot the same results with
    respect to the total number of parameters in the right plot.}
\label{fig:exp}
\end{figure}

In this section we show that the theoretical results of
sec.~\ref{sec:main_results} indeed hold in practice. In other words, there
exists tasks that require the highly expressive power of overlapping
architectures, on which non-overlapping architectures would have to grow
by an exponential factor to achieve the same level of performance. We
demonstrate this phenomenon on standard ConvNets with ReLU activations that
follow the same architecture that was outlined in
proposition~\ref{prop:common_case}, while varying the number of channels and the
size of the receptive field of the $B{\times}B$ ``conv'' layers. The only change
we made, was to replace the $2{\times}2$-``pooling'' layers of the convolutional
type, with the standard $2{\times}2$-max-pooling layers, and using the same
number of channels across all layers. This was done for the purpose of having
all the learned parameters located only at the (possibly) overlapping layers.
More specifically, the network has 5 blocks, each starting with a $B{\times}B$
convolution with $C$ channels, stride $1 {\times} 1$, and ReLU activation, and
then followed by $2 {\times} 2$ max-pooling layer. After the fifth
``conv-pool'', there is a final dense layer with 10 outputs and softmax
activations.

We train each of these networks for classification over the
CIFAR-10 dataset, with two types of data augmentation schemes: (i) spatial
augmentations, i.e. randomly translating (up to 3 pixels in each direction) and
horizontally flipping each image, and (ii) color augmentations following
\citet{Dosovitskiy:2014tu}, i.e. randomly adding a constant shift (at most
$\pm 0.3$) to the hue, saturation, and luminance, for each attribute separately,
and in addition randomly sampling a multiplier (in the range $[0.5, 1.5]$) just
to the saturation and luminance. Though typically data augmentation is only used
for the purpose of regularization, we employ it for the sole purpose of raising
the hardness of the regular CIFAR-10 dataset, as even small networks can already
overfit and effectively memorize its small dataset. We separately test both
the spatial and color augmentation schemes to emphasize that our empirical
results cannot be explained simply by spatial-invariance type arguments.
Finally, the training itself is carried out for 300 epochs with
ADAM~\citep{Kingma:2014us} using its standard hyper-parameters, at which point
the loss of the considered networks have stopped decreasing. We report the
training accuracy over the augmented dataset in fig.~\ref{fig:exp}, where for
each value of the receptive field $B$, we plot its respective training
accuracies for variable number of channels $C$. The source code for reproducing
the above experiments and plots can be found at \githuburl{OverlapsAndExpressiveness}.

It is quite apparent that the greater $B$ is chosen, the less channels are
required to achieve the same accuracy. Moreover, for the non-overlapping case of
$B{=}1$, more than 2048 channels are required to reach the same performance of
networks with $B {>} 2$ and just 64 channels under the spatial
augmentations~--~which means effectively exponentially more channels were
required. Even more so, under the color augmentations, we were not able to train
non-overlapping networks to reach even the smallest overlapping network
($B=2$ and $C=16$). In terms of total number of parameters, there is a clear
separation between the overlapping and the non-overlapping types, and we once
again see more than an order of magnitude increase in the number of parameters
between an overlapping and non-overlapping architectures that achieve similar
training accuracy. As a somewhat surprising result, though based only on our
limited experiments, it appears that for the same number of parameters, all
overlapping networks attain about the same training accuracy, suggesting perhaps
that having the smallest amount of overlapping already attain all the benefits
overlapping provides, and that increasing it further does not affect the
performance in terms of expressivity.

As final remark, we also wish to acknowledge the limitations of drawing
conclusions strictly from empirical experiments, as there could be alternative
explanations to these observations, e.g. the effects overlapping has on the
optimization process. Nevertheless, our theoretical results suggests this is
less likely the case.

\section{Discussion} \label{sec:discussion}

The common belief amongst deep learning researchers has been that depth is one
of the key factors in the success of deep networks~--~a belief formalized
through the depth efficiency conjecture. Nevertheless, depth is
one of many attributes specifying the architecture of deep networks, and each
could potentially be just as important. In this paper, we studied the effect
overlapping receptive fields have on the expressivity of the network, and found
that having them, and more broadly denser connectivity, results in an
exponential gain in the expressivity that is orthogonal to the depth.

Our analysis sheds light on many trends and practices in contemporary design
of neural networks. Previous studies have shown that non-overlapping
architectures are already universal~\citep{expressive_power}, and even have
certain advantages in terms of optimization~\citep{Brutzkus:2017wp}, and yet,
real-world usage of non-overlapping networks is scarce. Though there could be
multiple factors involved, our results clearly suggest that the main culprit is
that non-overlapping networks are significantly handicapped in terms of
expressivity compared to overlapping ones, explaining why the former are so
rarely used. Additionally, when examining the networks that are commonly used in
practice, where the majority of the layers are of the convolutional type with
very small receptive field, and only few if any fully-connected
layers~\citep{Simonyan:2014ws,Springenberg:2014tx,He:2016ib},
we find that though they are obviously overlapping, their overlapping degree is
rather low. We showed that while denser connectivity can increase
the expressive capacity, even in the most common types of modern architectures
already exhibit exponential increase in expressivity, without relying on
fully-connected layers. This could partly explain that somewhat surprising
observation, as it is probable that such networks are sufficiently expressive
for most practical needs simply because they are already in the exponential
regime of expressivity. Indeed, our experiments seems to suggests the same,
in which we saw that further increases in the overlapping degree beyond the most
limited overlapping case seems to have insignificant effects on performance~--~a 
conjecture not quite proven by our current work, but one we wish to investigate
in the future.

There are relatively few other works which have studied the role of receptive
fields in neural networks. Several empirical works~\citep{Li:2005kc,
Coates:2011wo,Krizhevsky:2012wl} have demonstrated similar behavior, showing
that the classification accuracy of networks can sharply decline as the degree
of overlaps is decreased, while also showing that gains from using very large
local receptive fields are insignificant compared to the increase in
computational resources. Other works studying the receptive fields of neural
networks have mainly focused on how to learn them from the
data~\citep{Coates:2011tl,Jia:2012uz}. While our analysis has no direct
implications to those specific works, it does lay the ground work for
potentially guiding architecture design, through quantifying the expressivity of
any given architecture. Lastly, \citet{Luo:2016vj} studied the \emph{effective
total receptive field} of different layers, a property of a similar nature to
our total receptive field, where they measure the the degree to which each input
pixel is affecting the output of each activation. They show that under common
random initialization of the weights, the effective total receptive field has a
gaussian shape and is much smaller than the maximal total receptive field. They
additionally demonstrate that during training the effective total receptive
field grows in size, and suggests that weights should be initialized such that
the initial effective receptive field is large. Their results strengthen our
theory, by showing that trained networks tend to maximize their effective
receptive field, taking full potential of their expressive capacity.

To conclude, we have shown both theoretically and empirically that overlapping
architectures have an expressive advantage compared to non-overlapping ones. Our
theoretical analysis is grounded on the framework of ConvACs, which we extend to
overlapping configurations. Though are proofs are limited to this specific case,
previous studies~\citep{generalized_decomp} have already shown that such results
could be transferred to standard ConvNets as well, using most of the same
mathematical machinery. While adapting our analysis accordingly is left for
future work, our experiments on standard ConvNets~(see sec.~\ref{sec:exp})
already suggest that the core of our results should hold in this case as well.
Finally, an interesting outcome of moving from non-overlapping architectures to
overlapping ones is that the depth of a network is no longer capped at
$\log_2 \left(\textit{input size} \right)$, as has been the case in the models
investigated by \citet{expressive_power}~--~a property we will examine in future
works

\newcommand{\acknowledgments}{This work is supported by Intel grant ICRI-CI \#9-2012-6133, by ISF Center grant 1790/12 and by the European Research Council (TheoryDL project).}
\ifdefined\COLT
	\acks{\acknowledgments}
\else
	\ifdefined\CAMREADY
		\subsubsection*{Acknowledgments}
		\acknowledgments
	\fi
\fi

\small{
\ifdefined\ICLR
\bibliographystyle{plainnat}
\fi
\ifdefined\NIPS
\bibliographystyle{plainnat}
\fi
\bibliography{refs.bib}
}

\clearpage
\appendix

\section{Background on Convolutional Arithmetic Circuits}\label{app:convac}

\begin{figure}
\centering
\includegraphics[width=0.9\linewidth]{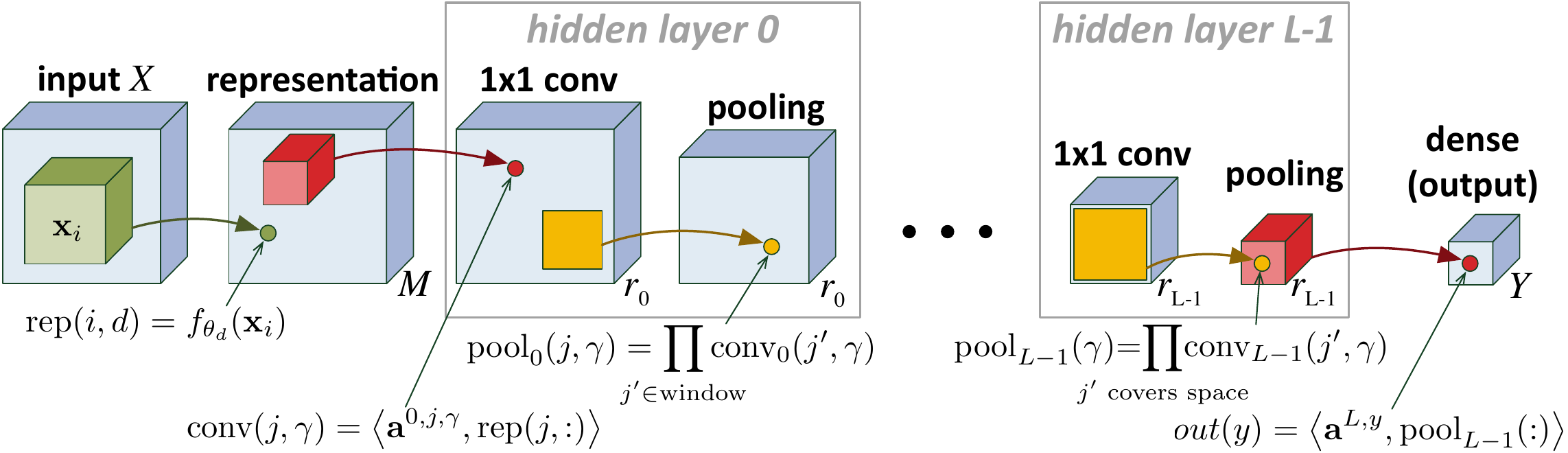}
\caption{The original Convolutional Arithmetic Circuits as presented by \citet{expressive_power}.}
\label{fig:original_convac}
\end{figure}

We base our analysis on the convolutional arithmetic circuit~(ConvAC)
architecture introduced by \citet{expressive_power}, which is illustrated by
fig.~\ref{fig:original_convac}, and can be simply thought of as a regular
ConvNet, but with linear activations and product pooling layers, instead of the
more common non-linear activations (e.g. ReLU) and average/max pooling. More
specifically, each point in the input space of the network, denoted by
$X=(\x_1,\ldots,\x_N)$, is represented as an $N$-length sequence of
$s$-dimensional vectors $\x_1,\ldots,\x_N \in \R^s$. $X$ is typically thought of
as an image, where each $\x_i$ corresponds to a local patches from that image.
The first layer of the network is referred to as the representation layer,
consisting of applying $M$ representation functions
$f_{\theta_1},\ldots,f_{\theta_M}:\R^s \to \R$ on each local patch $\x_i$,
giving rise to $M$ feature maps. Under the common setting, where the
representation functions are selected to be
$f_{\theta_d}(\x) = \sigma(\w_d^T\x + b_d)$ for some point-wise activation
$\sigma(\cdot)$ and parameterized by $\theta_d = (\w_d,b_d) \in \R^s \times \R$,
the representation layer reduces to the standard convolutional layer. Other
possibilities, e.g. gaussian functions with diagonal covariances,
have also been considered in \citet{expressive_power}. Following the
representation layer, are hidden layers indexed by $l=0,\ldots,L-1$, each begins
with a $1\times1$ \emph{conv} operator, which is just an $r_{l-1}\times1\times1$ 
convolutional layer with $r_{l-1}$ input channels and $r_l$ output channels,
with the sole exception that parameters of each kernel could be spatially
unshared  (known as locally-connected layer~\citep{Taigman:2014vs}). Following
each \emph{conv} layer is a spatial pooling, that takes products of
non-overlapping two-dimensional windows covering the output of the previous
layer, where for $l=L-1$ the pooling window is the size of the entire spatial
dimension (i.e. global pooling), reducing its output's shape to a
$r_{L-1}\times 1 \times 1$, i.e. an $r_{L-1}$-dimensional vector. The final $L$
layer maps this vector with a dense linear layer into the $Y$ network outputs,
denoted by $\h_y(\x_1,\ldots,\x_N)$, representing score functions classifying
each $X$ to one of the classes through:
$y^* = \argmax_y \h_y(\x_1,\ldots,\x_N)$. As shown in \citet{expressive_power},
these functions have the following form:
\begin{align}\label{eq:convac}
\h_y(\x_1,\ldots,\x_N) = \sum_{d_1,\ldots,d_N=1}^M \A^y_{d_1,\ldots,d_N} \prod_{i=1}^N f_{\theta_{d_i}}(\x_i)
\end{align}
where $\A^y$, called the \emph{coefficients tensor}, is a tensor of order $N$
and dimension $M$ in each mode, which for the sake of discussion can simply be
seen as a multi-dimensional array, specified by $N$ indices $d_1,\ldots,d_N$
each ranging in $\{1,\ldots,M\}$, with entries given by polynomials in the
network's \emph{conv} weights. A byproduct of eq.~\ref{eq:convac} is that for a
fixed set of $M$ representation functions, all functions represented by ConvACs
lay in the same subspace of functions.

\section{Comparison to Pooling Geometry}\label{app:pooling_geometry}

From theorem \ref{thm:main_overlaps} we learn that overlaps
give rise to networks which almost always cannot be
efficiently implemented by non-overlapping ConvAC with
standard pooling geometry. However, as proven by
\citet{inductive_bias}, a ConvAC that uses a different
pooling geometry \textendash{} i.e. the input to the pooling
layers are not strictly contiguous windows from the previous
layer \textendash{} also cannot be efficiently implemented
by the standard ConvAC with standard pooling geometry. This
raises the question of whether overlapping operations are
simply equivalent to a ConvAC with a different pooling
geometry and nothing more. We answer this question in two
parts. First, a ConvAC with a different pooling geometry
might be able to implement some function more efficiently
than ConvAC with standard pooling geometry, however, the
reverse is also true, that a ConvAC with standard pooling
can implement some functions more efficiently than ConvAC
with alternative pooling. In contrast, a ConvAC that uses
overlaps is still capable to implement efficiently any
function that a non-overlapping ConvAC with standard pooling
can. Second, we can also show that some overlapping
architectures are exponentially efficient than any
non-overlapping ConvAC regardless of its pooling geometry.
This is accomplished by first extending
lemma~\ref{lemma:mat_rank_bound} to this case:
\begin{lemma}\label{lemma:all_mat_ranks_bound}
    Under the same conditions as lemma~\ref{lemma:mat_rank_bound}, if for all
    partitions $P \cupdot Q$ such that $|P| = |Q| = \nicefrac{N}{2}$ it holds
    that $\rank{\mat{\A(\h_y)}_{P,Q}} \geq T$, then any non-overlapping ConvAC
    regardless of its pooling geometry must have at least $T$ channels in its
    next to last layer to induce the same grid tensor.
\end{lemma}
Next, in theorem \ref{thm:overlapping_overload} below show that some overlapping
architectures can induce grid tensors whose matricized rank is exponential for
any equal partition of its indices, proving they are indeed exponentially
more efficient:
\begin{theorem}\label{thm:overlapping_overload}
    Under the same settings as
    theorem~\ref{thm:main_overlaps}, consider a GC network
    whose representation layer is followed by a GC layer
    with local receptive field $H \times H$, stride
    $1 \times 1$, and $D \geq M$ output channels, whose
    parameters are ``unshared'', i.e. unique to each spatial
    location in the output of the layer as opposed to shared
    across them, followed by $(L-1)$ arbitrary GC layers,
    whose final output is a scalar. For any choice of the
    parameters, except a null set (with respect to the
    Lebesgue measure) and for any template vectors such that
    $F$ is non-singular, then the matricized rank of the
    induced grid tensor is equal to $M^{\frac{H^2}{2}}$,
    for any equal partition of the indices. The exact same
    result holds if the parameters of the first GC layers
    are ``shared'' and $D \geq M \cdot H^2$.
\end{theorem}
\begin{proof}[Proof sketch]
    We follow the same steps of our proof of theorem~\ref{thm:main_overlaps},
    however, we do not construct just one specific overlapping network that
    attains a rank of $D \geq M \cdot H^2$, for all possible matricizations of
    the induced grid tensor. Instead, we construct a separate network for each
    possible matricization. This proves that with respect to the Lebesgue
    measure over the network's parameters space, separately for each pooling
    geometry, the set of parameters for which the lower bound does not hold is
    of measure zero. Since a finite union of zero measured sets is also of
    measure zero, then the lower bound with respect to all possible pooling
    geometries holds almost everywhere, which concludes the proof sketch. See
    app.~\ref{app:proofs:overlapping_overload} for a complete proof.
\end{proof}
It is important to note that though the above theorem shows that pooling
geometry on its own is less expressive than overlapping networks with standard
pooling, it does not mean that pooling geometry is irrelevant. Specifically, we
do not yet know the effect of combining both overlaps and alternative pooling
geometries together. Additionally, many times sufficient expressivity is not the
main obstacle for solving a specific task, and the inductive bias induced by a
carefully chosen pooling geometry could help reduce overfitting.

\section{Deferred Proofs}
\label{app:proofs}

In this section we present our proofs for the theorems and claims stated in the
body of the article.

\subsection{Preliminaries}
\label{app:proofs:preliminaries}

In this section we lay out the preliminaries required to understand the proofs
in the following sections. We begin with a limited introduction to tensor
analysis, followed by quoting a few relevant known results relating tensors to
ConvACs.

We begin with basic definitions and operations relating to tensors. Let
$\A \in \R^{M_1 \otimes \cdots \otimes M_N}$ be a tensor of order $N$ and
dimension $M_i$ in each mode $i \in [N]$ (where $[N]=\{1,\ldots,N\}$), i.e.
$\A_{d_1,\ldots,d_N} \in \R$ for all $i \in [N]$ and $d_i \in [M_i]$. For
tensors $\A^{(1)}$ and $\A^{(2)}$ of orders $N^{(1)}$ and $N^{(2)}$, and
dimensions $M^{(1)}_{i_1}$ and $M^{(2)}_{i_2}$ in each of the modes
$i_1 \in [N^{(1)}]$ and $i_2 \in [N^{(2)}]$, respectively, we define their
tensor product $\A^{(1)} \otimes \A^{(2)}$ as the order $N^{(1)} + N^{(2)}$
tensor, where
\begin{align*}
    \left(\A^{(1)} \otimes \A^{(2)}\right)_{d_1,\ldots,d_{N^{(1)}+N^{(2)}}}
    = \A^{(1)}_{d_1,\ldots,d_{N^{(1)}}} \cdot \A^{(2)}_{d_{N^{(1)} + 1},\ldots,d_{N^{(1)} + N^{(2)}}} 
\end{align*}
For a set of vectors $\vv^{(1)}\in\R^{M_1},\ldots,\vv^{(N)}\in\R^{M_N}$, the $N$
ordered tensor $\A = \vv^{(1)}{\otimes}{\cdots}{\otimes}\vv^{(N)}$ is called an
\emph{elementary tensor}, or \emph{rank-1 tensor}. More generally,
any tensor can be represented as a linear combination of rank-1 tensors, i.e.
$\A = \sum_{z=1}^Z \vv^{(Z, 1)}{\otimes}{\cdots}{\otimes}\vv^{(Z, 1)}$, known as 
\emph{rank-1 decomposition}, or CP decomposition, where the minimal $Z$ for
which this equality holds is knows as the \emph{tensor rank} of $\A$. Given a
set of matrices
$F^{(1)}{\in}\R^{M'_1 \times M_1},\ldots,F^{(N)}{\in}\R^{M'_N \times M_N}$, we
denote by $\mathbf{F} = (F^{(1)}{\otimes}{\cdots}{\otimes}F^{(N)})$ the linear
transformation from $\R^{M_1{\otimes}{\cdots}{\otimes}M_N}$ to
$\R^{M'_1{\otimes}{\cdots}{\otimes}M'_N}$, such that for any elementary tensor
$\A$, with notations as above, it holds that:
\begin{align*}
    \mathbf{F}(\A)=F^{(1)}(\vv^{(1)}) \otimes \cdots \otimes F^{(N)}(\vv^{(N)})
\end{align*}
$\mathbf{F}(\A)$ is defined for a general tensor $\A$ through its rank-1
decomposition comprising elementary tensors and applying $\mathbf{F}$ on each of
them, which can be shown to be equivalent to 
\begin{align}\label{eq:kronecker_product_operator}
    \mathbf{F}(\A)_{k_1,\ldots,k_N} = \sum_{d_1=1}^{M_1}\cdots\sum_{d_N=1}^{M_N} \A_{d_1,\ldots,d_N} \prod_{i=1}^N F^{(i)}_{k_i,d_i}
\end{align}
A central concept in tensor analysis is that of \emph{tensor matricization}. Let
$P \cupdot Q=[N]$ be a disjoint partition of its indices, such that
$P = \{p_1, \ldots, p_{|P|}\}$ with $p_1 < \ldots < p_{|P|}$, and
$Q = \{q_1, \ldots, q_{|Q|}\}$ with $q_1 < \ldots < q_{|Q|}$. The matricization
of $\A$ with respect to the partition $P \cupdot Q$, denoted by
$\mat{\A}_{P,Q}$, is the
$\left(\prod_{t=1}^{|P|} M_{p_t}\right)$-by-$\left(\prod_{t=1}^{|Q|} M_{q_t}\right)$
matrix holding the entries of $\A$, such that for all $i \in [N]$ and
$d_i \in [M_i]$ the entry $A_{d_1, \ldots, d_N}$ is placed in row index
$1 + \sum_{t=1}^{|P|} (d_{p_t} - 1) \prod_{t'=t+1}^{|P|} M_{p_{t'}}$ and column
index
$1 + \sum_{t=1}^{|Q|} (d_{q_t} - 1) \prod_{t'=t+1}^{|Q|} M_{q_{t'}}$. Applying
the matricization operator $\mat{\cdot}_{P,Q}$ on the tensor product operator
results in the \emph{Kronecker Product}, i.e. for an $N$-ordered tensor $\A$, a
$K$-ordered tensor $\B$, and the partition $P\cupdot Q = [N+K]$, it holds that
\begin{align*}
    \mat{\A \otimes \B}_{P,Q} = \mat{\A}_{P\cap[N],Q\cap[N]} \odot \mat{\B}_{(P-N) \cap [K], (Q-N) \cap [K]}
\end{align*}
where $P-N$ and $Q-N$ are simply the sets obtained by subtracting the number $N$ 
from every element of $P$ or $Q$, respectively. In concrete terms, the Kronecker
product for the matrices $A \in \R^{M_1\times M_2}$ and $B\in\R^{N_1\times N_2}$
results in the matrix $A \odot B \in \R^{M_1 N_1 \times M_2 N_2}$ holding
$A_{ij}B_{kl}$ in row index $(i-1)N_1 + k$ and column index $(j-1)N_2 + l$. An
important property of the Kronecker product is that
$\rank{A \odot B} = \rank{A}\cdot \rank{B}$. Typically, when wish to compute
$\rank{\mat{\A}_{P,Q}}$, we will first decompose it to a Kronecker product of
matrices.

For a linear transform $\mathbf{F}$, as defined above in
eq.\ref{eq:kronecker_product_operator}, and a partition $P \cupdot Q$, if
$F^{(1)},\ldots,F^{(N)}$ are non-singular matrices, then $\mathbf{F}$ is
invertible and the matrix rank of $\mat{\A}_{P,Q}$ equals to the matrix rank of
$\mat{\mathbf{F}(\A)}_{P,Q}$ (see proof in \citet{Hackbusch-book}). Finally, we
define the concept of \emph{grid tensors}: for a function\\
${f:\R^s \times \cdots \times \R^s \to \R}$ and a set of \emph{template vectors} 
$\x^{(1)},\ldots,\x^{(M)} \in \R^s$, the $N$-order grid tensor $\A(f)$ is
defined by
$\left(\A(f)\right)_{d_1,\ldots,d_N} = f(\x^{(d_1)},\ldots,\x^{(d_N)})$.

In the context of ConvACs, circuits and the functions they can realize are
typically examined through the matricization of the grid tensors they induce.
The following is a succinct summary of the relevant known results used in our
proofs~--~for a more detailed discussion, see previous works
\citep{expressive_power,generalized_decomp,inductive_bias}. Using the same
notations from eq.~\ref{eq:convac} describing a general ConvAC, let $\A^y$ be
the coefficients tensor of order $N$ and dimension $M$ in each mode, and let
$f_{\theta_1},\ldots,f_{\theta_M}{:}\R^s{\to}\R$ be a set of $M$ representation
functions (see app.~\ref{app:convac}). Under the above definitions, a
non-overlapping ConvAC can be said to decompose the coefficients tensor $\A^y$.
Different network architectures correspond to known tensor decompositions:
shallow networks corresponds to rank-1 decompositions, and deep networks
corresponds to Hierarchical Tucker decompositions.

In \citet{inductive_bias}, it was found that the matrix rank of the
matricization of the coefficients tensors $\A^y$ could serve as a bound for the
size of networks decomposing $\A^y$. For the conventional non-overlapping ConvAC
and the contiguous ``low-high'' partition
$P = \{1,\ldots,\nicefrac{N}{2}\}, Q=\{\nicefrac{N}{2} + 1, \ldots, N\}$ of
$[N]$, the rank of the matricization $\mat{\A^y}_{P,Q}$ serves as a lower-bound
on the number of channels of the next to last layer of any network which
decomposes the coefficients tensor $\A$. In the common case of square inputs,
i.e. the input is of shape $H \times H$ and $N = H^2$, it is more natural to
represent indices by pairs $(j,i)$ denoting the spatial location of each
``patch'' $\x_{(j,i)}$, where the first argument denotes the vertical location
and the second denotes the horizontal location. Under such setting the
equivalent ``low-high'' partitions are either the ``left-right'' partition, i.e.
$P=\{(j,i)| j \leq \frac{H}{2}\}, Q=\{(j,i)| j > \frac{H}{2}\}$, or the
``top-bottom'' partition,  i.e.
$P=\{(j,i)| i \leq \frac{H}{2}\}, Q=\{(j,i)| i > \frac{H}{2}\}$. More generally, 
when considering networks using other \emph{pooling geometries}, i.e. not
strictly contiguous pooling windows, then for each pooling geometry there exists
a corresponding partition $P \cupdot Q$ such that $\rank{\mat{\A^y}_{P,Q}}$
serves as its respective lower-bound.

Though the results in \citet{inductive_bias} are strictly based on the
matricization rank of the coefficients tensors, they can be transferred to the
matricization rank of grid tensors as well. Grid tensors were first considered
for analyzing ConvACs in \citet{generalized_decomp}. For a set of $M$ template
vectors $\x^{(1)},\ldots,\x^{(M)}\in\R^s$, we define the matrix
$F \in \R^{M \times M}$ by $F_{ij} = f_{\theta_j}(\x^{(i)})$. With the above
notations in place, we can write the the grid tensor $\A(\h_y)$ for the function
$\h_y(\x_1,\ldots,\x_N)$ as:
\begin{align*}
    \A(\h_y)_{k_1,\ldots,k_N} &= \h_y(\x^{(k_1)},\ldots,\x^{(k_N)})\\
    &= \sum_{d_1,\ldots,d_N=1}^M \A^y_{d_1,\ldots,d_N}
        \prod_{i=1}^N f_{\theta_{d_i}}(\x^{(k_i)})\\
    &= \sum_{d_1,\ldots,d_N=1}^M \A^y_{d_1,\ldots,d_N}
        \prod_{i=1}^N F_{k_i,d_i} \\
    \Rightarrow \A(\h_y) &= (F \otimes \cdots \otimes F)(\A^y)
\end{align*}
If the representation functions are linearly independent and continuous, then we
can choose the template vectors such that $F$ is non-singular (see
\citet{generalized_decomp}), which according to the previous discussion on
tensor matricization, means that for any partition $P \cupdot Q$ and any
coefficients tensor $\A^y$, it holds that
$\rank{\mat{\A^y}_{P,Q}} = \rank{\mat{\A(\h_y)}_{P,Q}}$. Thus, any lower bound
on the matricization rank of the grid tensor translates to a lower bound on the
matricization rank of the coefficients tensors, which in turn serves as lower
bound on the size of non-overlapping ConvACs. The above discussion leads to the
proof of lemma~\ref{lemma:mat_rank_bound} and
lemma~\ref{lemma:all_mat_ranks_bound} that were previously stated:
\begin{proof}[Proof of lemma~\ref{lemma:mat_rank_bound} and
              lemma~\ref{lemma:all_mat_ranks_bound}]
    For the proofs of the base results with respect to the coefficients tensor,
    see \citet{inductive_bias}. To prove it is possible to choose the template
    vectors such that $F$ is non-singular, see \citet{generalized_decomp}. To
    prove that if $F$ is non-singular, then the grid tensor and the coefficients
    tensor have the same matricization rank, see lemma 5.6 in
    \citet{Hackbusch-book}.
\end{proof}

We additionally quote the following lemma regarding the prevalence of the
maximal matrix rank for matrices whose entries are polynomial functions:
\begin{lemma}\label{lemma:rank_everywhere}
    Let $M, N, K \in \N$, $1 \leq r \leq \min\{M,N\}$ and a polynomial mapping
    $A:\R^K \to \R^{M \times N}$, i.e. for every $i \in [M]$ and $j\in [N]$ it
    holds that $A_{ij}:\R^k \to \R$ is a polynomial function. If there exists a
    point $\x \in \R^K$ such that $\rank{A(\x)} \geq r$, then the set
    $\{\x \in \R^K | \textrm{rank}{A(\x)} < r\}$ has zero measure (with respect
    to the Lebesgue measure over $\R^K$).
\end{lemma}
\begin{proof}
    See \citet{tmm}.
\end{proof}

Finally, we simplify the notations of the GC layer with product pooling function
for the benefit of following our proofs below. We will represent the parameters
of the $l$'th GC layer by $\{(\w^{(l,c)} \in \R^{D^{(l-1)} \times R^{(l)}
\times R^{(l)}},\bb^{(l,c)}\in \R^{R^{(l)} \times R^{(l)}})\}_{c\in[D^{(l)}]}$,
where $\w^{(l,c)}$ represents the weights and $\bb^{(l,c)}$ the biases. Let
$X \in \R^{H^{(in)} \times H^{(in)} \times D^{(in)}}$ be the input to the layer
and $Y \in \R^{D^{(out)} \times H^{(out)} \times H^{(out)}}$ be the output, then
the following equality holds:
\begin{equation}
    Y_{c,u,v} = \prod_{j=1}^{R^{(l)}} \prod_{i=1}^{R^{(l)}}
        \left(b_{ji}^{(l,c)} +
        \sum_{d=1}^{D^{(in)}} w_{dji}^{(l,c)} X_{d,u S^{(l)} + j, v S^{(l)} + i} \right)
\end{equation}
The above treats the common case where the parameters are shared across all
spatial locations, but sometimes we wish to consider the ``unshared'' case, in
which there is are different weights and biases for each location, which we
denote by $\{(\w^{(l,c,u,v)} \in \R^{D^{(l-1)} \times R^{(l)}
\times R^{(l)}},\bb^{(l,c,u,v)}\in \R^{R^{(l)} \times R^{(l)}})\}_{c\in[D^{(l)}], u \in [H^{(out)}], v \in [H^{(out)}]}$.

With the above definitions and lemmas in place, we are ready to prove the
propositions and theorems from the body of the article.

\subsection{Proof of Proposition~\ref{prop:nothing_to_lose}}
\label{app:proofs:nothing_to_lose}

Proposition~\ref{prop:nothing_to_lose} is a direct corollary of the following
two claims:
\begin{claim} \label{claim:any_window}
    Let $f{:}\R^{D^{(in)} \times H^{(in)} \times H^{(in)}} \to
    \R^{D^{(out)} \times H^{(out)} \times H^{(out)}}$
    be a function realized by a single GC layer with $R \times R$ local
    receptive field, $S \times S$ stride, and $D^{(out)}$ output channels, that
    is parameterized by $\{ (\w^{(c)}, \bb^{(c)})\}_{c=1}^{D^{(out)}}$. For all
    $\tilde{R}\geq R$, a GC layer with $\tilde{R}{\times}\tilde{R}$ local
    receptive field, $S{\times}S$ stride, and $C$ output channels, parameterized
    by $\{ (\tilde{\w}^{(c)}, \tilde{\bb}^{(c)})\}_{c=1}^{D^{(out)}}$, could
    also realize $f$. The same is true for the unshared case of both layers.
\end{claim}
\begin{proof}
    The claim is trivially satisfied by setting $\tilde{\w}^{(c)}$ such that it
    is equal to $\w^{(c)}$ in all matching coordinates, while using zeros for
    all other coordinates. Similarly, we set $\tilde{\bb}^{(c)}$ to be equal to
    $\bb^{(c)}$ in all matching coordinates, while using ones for all other
    coordinates.
\end{proof}
\begin{claim} \label{claim:identity}
    Let $f{:}\R^{D^{(in)} \times H^{(in)} \times H^{(in)}} \to
    \R^{D^{(out)} \times H^{(out)} \times H^{(out)}}$ be a function realized by
    a GC layer with $R{\times}R$ local receptive field and $1{\times}1$ stride,
    parameterized by $\{ (\w^{(c)}, \bb^{(c)})\}_{c=1}^{D^{(out)}}$. Then there
    exists an assignment to $(\w,\bb)$ such that $f$ is the identity function
    $f(X) = X$. The same is true for the unshared case of both layers.
\end{claim}
\begin{proof}
From claim~\ref{claim:any_window} it is sufficient to show the above
holds for $R=1$. Indeed, setting $$w_{d}^{(c)}{=}1_{\left[d{=}c\right]}{=}\begin{cases}1 & d=c \\ 0 & d\neq c\end{cases}$$
and $\bb^{(c)} \equiv \0$ satisfies the claim.
\end{proof}

\subsection{Proof of Theorem~\ref{thm:main_overlaps}}
\label{app:proofs:main_overlaps}

We wish to show that for all choices of parameters, except a null set (with
respect to Lebesgue measure) the grid tensor induced by the given GC network
has rank satisfying eq.~\ref{eq:lower_bound}. Since the entries of the
matricized grid tensor are polynomial function of its parameters, then according
to lemma~\ref{lemma:rank_everywhere}, it is sufficient to find a single example
that achieves this bound. Hence, our proof is simply the construction of such an
example.

Recall that the template vectors must hold that for the matrix $F$, defined by
$F_{ij} = f_j(\x^{(i)})$, where $\{f_j\}_{j=1}^M$ are the representation
matrices, is a non-singular matrix. We additionally assume in the following
claims that the output of the representation layer is of width and height equal
to $H \in \N$, where $H$ is an even number~--~the claims and proofs can however
be easily adapted to the more general case.

Assume a ConvAC as described in the theorem, with representation layer defined
according to above, followed by $L$ GC layers, where the $l$'th layer has a
local receptive field of $R^{(l)}$, a stride of $S^{(l)}$, and $D^{(l)}$
output channels. We first construct our example, that achieves the desired
matricization rank, for the simpler case where the first layer following the
representation layer has a local receptive field large enough, i.e. when it is
larger than $\frac{H}{2}$. Recall that for the first layer the total receptive
field is equal to its local receptive field. In the context of
theorem~\ref{thm:main_overlaps}, this first layer satisfies the conditions
necessary to produce the lower bound given in the theorem.

The specific construction is presented in the following claim, which relies on
utilizing the large local receptive field to match each spatial location in the
left side of the input with one from the right side, such that for each such
pair, the respective output of the first layer will represent a mostly diagonal
matrix. We then set the rest of the parameters such that the output of the
entire network is defined by a tensor product of mostly diagonal matrices. Since
the matricization rank of the tensor product of matrices is equal to the product
of the individual ranks, it results in an exponential form of the rank as is
given in the theorem.

\begin{claim}\label{claim:one_overlap}
    Assume a ConvAC as defined above, ending with a single scalar output. For
    all $l\in[L]$, the parameters of the $l$-th GC layer are denoted by
    $\{(\w^{(l,c)} \in \R^{D^{(l-1)} \times R^{(l)} \times R^{(l)}},
    \bb^{(l,c)}\in \R^{R^{(l)} \times R^{(l)}})\}_{c\in[D^{(l)}]}$. Let
    $\h(\x_1,\ldots,\x_N)$ be the function realized the output of network.
    Additionally define $R\equiv R^{(1)}$, $S \equiv S^{(1)}$ and
    $D \equiv \min \{D^{(1)}, M\}$. If $R > \frac{H}{2}$,
    and the weights $\w^{(1,c)}$ and biases $\bb^{(1,c)}$ of the first GC layer
    layer are set to:
    \begin{align*}
        w_{mji}^{(1,c)} & =\begin{cases}
            -\alpha\left(F^{-1}\right)_{m,c} & c \leq D \tand (j,i)\in \{(1,1), (\rho, \tau)\} \\
            0 & \text{Otherwise}
        \end{cases}\\
        b_{ji}^{(1,c)} & =\begin{cases}
            \beta &  c \leq D \tand (j,i) \in \{(1,1), (\rho, \tau)\} \\
            1 &  c \leq D \tand (j,i) \not\in \{(1,1), (\rho, \tau)\} \\
            0 & \text{Otherwise}
        \end{cases}
    \end{align*}
    where $\beta=\frac{2\alpha}{D}$, then there exists an assignment to $\alpha$
    and the parameters of the other GC layers such that\\
    $\rank{\mat{\A(\h)}_{P,Q}}= D^{\left\lfloor \frac{H - R}{S} + 1
    \right\rfloor \cdot \fracceil{H}{S}}$,
    where $P\cupdot Q$ is either the ``left-right'' or ``top-bottom'' partition,
    and $(\rho,\tau)$ equals to $(1,R)$ or $(R,1)$, respectively.
\end{claim}
\begin{proof}
The proof for either the ``left-right'' or ``top-bottom'' partition is completely symmetric, thus it is enough
to prove the claim for the ``left-right'' case, where $(\rho, \tau) = (1,R)$. We wish to compute the entry
$\A(\h)_{d_{(1,1)},\ldots,d_{(H,H)}}$ of the induced grid tensor for arbitrary indices $d_{(1,1)},\ldots,d_{(H,H)}$.
Let $O \in \R^{M \times H \times H}$ be the 3-order tensor output of the representation layer, where
$O_{m,j,i} = F_{d_{(j,i)},m}$ for the aforementioned indices and for all $1 \leq i,j \leq H$ and $m \in [M]$.

We begin by setting the parameters of all layers following the first GC layer, such that they are equal to
computing the sum along the channels axis of the output of the first GC layer, followed by a global product
of all of the resulting sums. To achieve this, we can first assume w.l.o.g. that these layers are non-overlapping
through proposition~\ref{prop:nothing_to_lose}. We then set the parameters of the second GC layer to
$\w^{(2,c)} = \mathbf{1}$ and $\bb^{(2,c)} \equiv \mathbf{0}$, i.e. all ones and all zeros, respectively, which is
equivalent to taking the sum along the channels axis for each spatial location, followed by taking the products
over non-overlapping local receptive fields of size $R^{(2)}$. For the other layers, we simply set them as
to take just the output of the first channel of the output of the preceding layer, which is equal to setting
their parameters to $w^{(l,c)}_{dji} = \begin{cases} 1 & d = 1\\0 & d \neq 1\end{cases}$ and $\bb^{(l,c)} \equiv 0$.
Setting the parameters as described above results in:
\begin{align}\label{eq:claim:one_overlap}
\left(\A(\h)\right)_{d_{(1,1)},\ldots,d_{(H,H)}} &= \prod_{\substack{0\leq u S < H \\ 0 \leq v S < H}}\sum_{c=1}^{D^{(1)}}\prod_{j,i=1}^R
\underbrace{\left(b_{ji}^{(1,c)}+\sum_{m=1}^M w_{mji}^{(1,c)}O_{m, u S + j, v S + i}\right)}_{g(u,v,c,j,i)}
\end{align}
where we extended $O$ with zero-padding for the cases where $uS+j > H$ or $vS+i > H$, as
mentioned in sec.~\ref{sec:overlapping_convac}. Next, we go through the technical process of reducing eq.~\ref{eq:claim:one_overlap}
to a product of matrices.

Substituting the values of $w_{dji}^{\left(1,c\right)}$ and $b_{ji}^{\left(1,c\right)}$ with those defined in the claim,
and computing the value of $g(u,v,c,j,i)$ results in:
\begin{align*}
g(u,v,c,j,i) &= \begin{cases}
\beta-\alpha\sum_{m=1}^M\left(F^{-1}\right)_{m,c} F_{d_{(u S + j,v S + i)},m} &  c \leq D \tand v S + R \leq H \tand (j,i) \in \{(1,1),(1,R)\} \\
\beta-\alpha\sum_{m=1}^M\left(F^{-1}\right)_{m,c} F_{d_{(u S + j,v S + i)},m} &  c \leq D \tand v S + R > H \tand (j,i) = (1,1) \\
\beta &  c \leq D \tand v S + R > H \tand (j,i) = (1,R) \\
1 &  c \leq D \tand (j,i) \not\in \{(1,1),(1,R)\} \\
0 & c > D
\end{cases}\\
&= \begin{cases}
\beta - \alpha \left(F F^{-1}\right)_{d_{(u S + 1, v S + i)},c} &  c \leq D \tand v S + R \leq D \tand (j,i) \in \{(1,1),(1,R)\} \\
\beta - \alpha \left(F F^{-1}\right)_{d_{(u S + 1, v S + i)},c} &  c \leq D \tand v S + R > H \tand (j,i) = (1,1) \\
\beta & c \leq D \tand v S + R > H \tand (j,i) = (1,R) \\
1 & c \leq D \tand (j,i) \not\in \{(1,1),(1,R)\} \\
0 & c > D
\end{cases}
\end{align*}
from which we derive:
\begin{align*}
f(u,v) \equiv \sum_{c=1}^{D^{(1)}}\prod_{j,i=1}^R g(u,v,c,j,i) &= \begin{cases}
\begin{aligned}\,&D\beta^2 - \alpha \beta(1_{[d_{(uS + 1, vS + 1)} \leq D]} + 1_{[d_{(uS + 1, vS + R)} \leq D]}) \\\,& + \alpha^2 1_{[d_{(uS + 1, vS + 1)} = d_{(uS + 1, vS + R)} \leq D]}\end{aligned}&  v S + R \leq H \\
D\beta^2 - \alpha \beta 1_{[d_{(uS + 1, vS + 1)} \leq D]} &  v S + R > H
\end{cases}
\end{align*}
where $\left(\A(\h)\right)_{d_{(1,1)},\ldots,d_{(H,H)}}=\prod_{\substack{0\leq u S < H \\ 0 \leq v S < H}} f(u,v)$.

At this point we branch into two cases. If $S$ divides $R-1$, then for all $u,v \in \N$ such that $vS + R \leq H$ and $uS < H$,
the above expression for $f(u,v)$ and $f(u,v + \frac{R-1}{S})$ depends only on the indices $d_{(uS + 1, vS + 1)}$ and
$d_{(uS + 1, vS + R)}$, while these two indices affect only the aforementioned expressions. By denoting\\
$A^{(u,v)}_{d_{(uS + 1, vS + 1)},d_{(uS + 1, vS + R)}} = f(u,v)\cdot f(u,v + \frac{R-1}{S})$, we can write it as:
\begin{align*}
A_{ij}^{(u,v)} = \begin{cases}
\left(D\beta^2 -2\alpha\beta + \alpha^2 1_{[i=j]}\right)\left(D\beta^2 - \alpha\beta\right) & i,j \leq D \\
\left(D\beta^2 -\alpha\beta\right)\left(D\beta^2\right) & i \leq D \tand j > D \\
\left(D\beta^2 -\alpha\beta\right)^2 & i > D \tand j \leq D \\
\left(D\beta^2\right)^2 & i,j > D
\end{cases}
\end{align*}
where $i,j\in[M]$ stand for the possible values of $d_{(uS + 1, vS + 1)}$ and $d_{(uS + 1, vS + R)}$,
respectively. Substituting $\beta=\frac{2\alpha}{D}$, as stated in the claim, and
setting $\alpha=\left(\frac{D}{2}\right)^{\nicefrac{1}{4}}$, results in:
\begin{align*}
A_{ij}^{(u,v)} = \begin{cases}
1_{[i=j]} & i,j \leq D \\
\frac{4}{D} & i \leq D \tand j > D \\
\frac{2}{D} & i > D \tand j \leq D \\
\frac{8}{D} & i,j > D
\end{cases}
\end{align*}
which means $\rank{A^{(u,v)}}=D$. Since $\left(\A(\h)\right)_{d_{(1,1)},\ldots,d_{(H,H)}}$
equals to the product \\$\prod_{\substack{0\leq u S < H \\ 0 \leq v S \leq H - R}} A^{(u,v)}_{d_{(uS + 1, vS + 1)}, d_{(uS + 1, vS + R)}}$,
then $\mat{\A(\h)}_{P,Q}$ equals to the Kronecker product of the matrices in $\{A^{(u,v)}|0\leq u S < H, 0 \leq vS \leq H - R\}$,
up to permutation of its rows and columns, which do not affect its matrix rank. Thus, the matricization rank of $\A(\h)$ satisfies:
\begin{align*}
\rank{\mat{\A(\h)}_{P,Q}} &= D^{\abs{\{A^{(u,v)}|0\leq u S < H, 0 \leq vS \leq H - R\}}}
=  D^{\left\lfloor \frac{H - R}{S} + 1 \right\rfloor \cdot \fracceil{H}{S}}
\end{align*}
which proves the claim for this case.

If $S$ does not divide $R-1$, then for all $u,v\in\N$, such that $vS + R \leq H$ and $uS < H$, it holds
that $f(u,v)$ depends only on the indices $d_{(uS + 1, vS + 1)}$ and $d_{(uS + 1, vS + R)}$, and they
affect only $f(u,v)$. Additionally, for all $u,v\in\N$, such that $H < vS + R$, $vS < H$ and $uS < H$, it holds
that $f(u,v)$ depends only on the index $d_{(uS + 1, vS + 1)}$, and this index affects only $f(u,v)$.
Let us denote $\A^{(u,v)}_{d_{(uS + 1, vS + 1)},d_{(uS + 1, vS + R)}} = f(u,v)$ for $vS + R \leq H$:
\begin{align*}
A_{ij}^{(u,v)} = \begin{cases}
\left(D\beta^2 -2\alpha\beta + \alpha^2 1_{[i=j]}\right) & i,j \leq D \\
\left(D\beta^2 -\alpha\beta\right) & i \leq D \tand j > D \\
\left(D\beta^2 -\alpha\beta\right) & i > D \tand j \leq D \\
\left(D\beta^2\right) & i,j > D
\end{cases}
\end{align*}
which by setting $\beta=\frac{2\alpha}{D}$ and $\alpha=1$ results in:
\begin{align*}
A_{ij}^{(u,v)} = \begin{cases}
1_{[i=j]} & i,j \leq D \\
\frac{2}{D} & i \leq D \tand j > D \\
\frac{2}{D} & i > D \tand j \leq D \\
\frac{4}{D} & i,j > D
\end{cases}
\end{align*}
which means $\rank{A^{(u,v)}}=D$. For $vS + R > H$, and we can define
the vector $\aaa^{(u,v)}_{d_{(uS + 1, vS + 1)}} = f(u,v)$, which by using the same values of
$\alpha$ and $\beta$ results in:
\begin{align*}
a^{(u,v)}_i &= \begin{cases}
\frac{2}{D}& i \leq D \\
\frac{4}{D} & i > D
\end{cases}
\end{align*}
By viewing $\aaa^{(u,v)}$ as either a column or row vector, depending on whether $d_{(uS + 1, vS + 1)} \in P$
or \\$d_{(uS + 1, vS + 1)}~\in~Q$, respectively, it holds that $\mat{\A(\h)}_{P,Q}$ equals to the Kronecker product
of the matrices in $\{A^{(u,v)}\}_{\substack{0\leq u S < H\\ 0 \leq vS \leq H - R}} \cup
\{\aaa^{(u,v)}\}_{\substack{0\leq u S < H\\ H -R < vS < H}}$, up to permutations of its rows and columns,
which do not affect the rank. Since $\aaa^{(u,v)}\neq\mathbf{0}$ then
$\rank{\aaa^{(u,v)}}=1$, which means the matricization rank of $\A(\h)$ once again holds:
\begin{align*}
\rank{\mat{\A(\h)}_{P,Q}} &= D^{\abs{\{A^{(u,v)}|0 \leq u S < H, 0 \leq v S \leq H - R\}}}
=  D^{\left\lfloor \frac{H - R}{S} + 1 \right\rfloor \cdot \fracceil{H}{S}}
\end{align*}
\end{proof}

In the preceding claim we have describe an example for the case where the total receptive of the first GC
layer is already large enough for satisfying the conditions of the theorem. In the following claim we extend
this result for the general case. This is accomplished by showing that a network comprised of just $L$ GC
layers with local receptive fields $\{R^{(l)} \times R^{(l)} \}_{l\in[L]}$, strides $\{S^{(l)} \times S^{(l)}\}_{l\in[L]}$, and output channels
$\{D^{(l)}\}_{l\in[L]}$, can effectively compute the same output as the first GC layer from claim~\ref{claim:one_overlap},
for all inputs~\footnote{Notice that in this context, there is no representation layer, and the input can be any 3-order tensor.}.

Recall that the layer from claim~\ref{claim:one_overlap} performs an identical transformation on each $M \times 1 \times 1$
patch from its input, followed by taking the point-wise product of far-away pairs of transformed patches. Thus, the motivation
behind the specific construction we use, is to use the first of the $L$ layers to perform this transformation, while using half of its
output channels for storing the transformed patch from the same location, and the other half for storing a transformed patch, but
from a location farthest to the right, constrained by its local receptive field. This is equal to having one set of transformed patches
sitting still, while another ``shifted'' set of transformed patches. The other layers simply pass the the first half of the channels as is,
using an identity operation as defined in claim~\ref{claim:identity}, while continuously shifting the other half of the channels more
and more to the left, bringing faraway patches closer together. Finally, at the last layer we take both halves and multiple them together.

\begin{claim}\label{claim:many_overlaps}
    Assume $\Phi$ is a ConvAC comprised of just $L$ GC layers as described
    above, where the output of the network is not limited to a scalar value.
    Assume the total stride of the $L$-th GC layer is greater than
    $\nicefrac{H}{2}$, and let $\totstr^{(L)}$ and $\totrec^{(L,\alpha)}$ be the
    total stride and the $\alpha$-minimal total receptive field, respectively,
    for $\alpha = \nicefracfloor{H}{2} +1$. Let $\Psi$ be a ConvAC comprised of
    a single GC layer with local receptive field
    $R \equiv \totrec^{(L,\alpha)}$, stride $S \equiv \totstr^{(L)}$, output
    channels $D \equiv \min\{\frac{1}{2}\min_{1\leq l<L} D^{(l)},D^{(L)}\}$,
    where its weights and biases are set to the following:
    \begin{align*}
    w_{mji}^{(c)} & =\begin{cases}
    A_{m,c} & \left(j,i\right)\in\left\{ \left(1,1\right),\left(\rho, \tau\right)\right\} \\
    0 & \textrm{Otherwise}
    \end{cases}\\
    b_{ji}^{(c)} & =\begin{cases}
    \beta_c & \left(j,i\right)\in\left\{ \left(1,1\right),\left(\rho, \tau \right)\right\} \\
    1 & \textrm{Otherwise}
    \end{cases}
    \end{align*}
    for $\beta\in\R^D$, $A\in\R^{M\times D}$ and $(\rho, \tau) \in \{(1,R), (R, 1)\}$.
    Then, there exists a set of weights to the layers of~$\Phi$ such that for every input $X$, the output of~$\Phi$
    is equivalent to the output of $\Psi$ for channels $\leq D$, and zero otherwise.
\end{claim}
\begin{proof}
The two possible cases for $(\rho, \tau)$ are completely symmetric, thus it is enough to prove the
claim just for $(\rho, \tau) = (1, R)$. Additionally, we can assume w.l.o.g. that $\forall l, R^{(l)} > 1$,
by setting any $1 \times 1$ layer to act as pass-through according to claim~\ref{claim:identity},
and also assume that the $\alpha$-minimal total receptive field is exactly equal to the total receptive field
of the $L$-th layer, by applying claim~\ref{claim:any_window} to realize an equivalent network
with smaller windows. Finally, the case for $L = 1$ is trivial, and thus we assume $L > 1$.

Let us set the parameters $\left\{ \w^{\left(l,k\right)},\bb^{\left(l,k\right)}\right\} $
of the layers of $\Phi$ as follows:
\begin{align*}
w_{dji}^{\left(l,k\right)} & =\begin{cases}
- A_{d, k}                    & (l=1)        \tand (1 \leq k \leq D )    \tand  (j,i) = (1, 1)\\
- A_{d, k-D}                & (l=1)        \tand  (D   <  k \leq 2D)  \tand  (j,i) = (1, R^{(l)})\\
1_{\left[d=k\right]}      & (1 < l < L) \tand (1 \leq k \leq D)     \tand  (j,i) = (1, 1)\\
1_{\left[d=k\right]}      & (1 < l < L) \tand (D   <  k \leq 2D)   \tand  (j,i) = (1, R^{(l)})\\
1_{\left[d=k\right]}      & (l=L)         \tand (1 \leq k \leq D)     \tand  (j,i) = (1, 1)\\
1_{\left[d=M+k\right]} & (l=L )        \tand (1 \leq k \leq D)     \tand  (j,i) = (1, R^{(l)})\\
0                                & \text{Otherwise}
\end{cases}\\
b_{ji}^{\left(l,k\right)} & =\begin{cases}
\beta_k       & (l=1)        \tand (1 \leq k \leq D )    \tand  (j,i) = (1, 1)\\
\beta_{k-D} & (l=1)        \tand  (D   <  k \leq 2D)  \tand  (j,i) = (1, R^{(l)})\\
0                 & (1 < l < L) \tand (1 \leq k \leq D)     \tand  (j,i) = (1, 1)\\
0                 & (1 < l < L) \tand (D   <  k \leq 2D)   \tand  (j,i) = (1, R^{(l)})\\
0                 & (l=L)         \tand (1 \leq k \leq D)     \tand  (j,i) = (1, 1)\\
0                 & (l=L )        \tand (1 \leq k \leq D)     \tand  (j,i) = (1, R^{(l)})\\
1                 & \text{Otherwise}
\end{cases}
\end{align*}
It is left to prove the above satisfies the claim. Let $O^{(l)} \in \R^{D^{(l)} \times H^{(l)} \times H^{(l)}}$ be
the output of the $l$-th layer, for $l\in[0,\ldots,L]$, where $H^{(l)}$ is the width and height of the output.
We additionally assume that for indices beyond the bounds of $[D^{(l)}] \times [H^{(l)}] \times [H^{(l)}] $ the
value of $O^{(l)}$ is zero, i.e. we assume zero padding when applying the convolutional operation of the GC layer.
We extend the definition for $l=0$, by setting $D^{(0)} \equiv M$ and $H^{(0)} \equiv H$, where we identify
$O^{(0)}$ with the input to the network $\Phi$. Given the above, the output of the first layer for $k\in[D^{(1)}]$ and
$0\leq u,v < H^{(1)}$, is as follows:
\begin{align*}
O_{k,u+1,v+1}^{(1)} & =\prod_{j,i=1}^{R^{(1)}} \left(b_{ji}^{(1,k)}+\sum_{d=1}^{D^{(0)}} w_{dji}^{(1,k)}O^{(0)}_{d, u S_h + j, v S_w + i}\right)\\
 & =\begin{cases}
\beta_k + \sum_{d=1}^{M}A_{d, k} \cdot O^{(0)}_{d,u S_h + 1, v S_w + 1}                           & 1 \leq k \leq D\\
\beta_{k-D} + \sum_{d=1}^{M}A_{d, k-D} \cdot O^{(0)}_{d,u S_h + 1, v S_w + R^{(1)} }  & D < k \leq 2D\\
1                                                                                                                      & \textrm{Otherwise}
\end{cases}
\end{align*}
We will show by induction that for $1 < l < L$, $k \in [D^{(l)}]$ and $0 \leq u,v < H^{(l)}$ the output of the $l$-th
layer $O_{k,u+1,v+1}^{(l)}$ always equals to:
\begin{align*}
O_{k,u+1,v+1}^{(l)} &=\begin{cases}
O_{k, u \eta^{(l)} + 1, v \eta^{(l)} + 1}^{1} & k \leq D\\
O_{k, u \eta^{(l)} + 1, v \eta^{(l)} + \xi^{(l)}}^{1} & D < k \leq 2D\\
1 & \textrm{Otherwise}
\end{cases}
\end{align*}
where $\eta^{(l)} = \prod_{i=2}^l S^{(i)}$ and $\xi^{(l)} = R^{(l)}\cdot\eta^{(l-1)} + \sum_{k=2}^{l-1} (R^{(k)}- S^{(l)})\cdot \eta^{(k-1)}$.
It is trivial to verify that for $l=2$ it indeed holds, since:
\begin{align*}
O_{k,u+1,v+1}^{(2)} =\begin{cases}
O_{k, u S^{(2)} + 1, v S^{(2)} + 1}^{(1)} & k\leq D\\
O_{k, u S^{(2)} + 1,v S^{(2)} + R^{(2)}}^{(1)} & D<k\leq2D\\
1 & \textrm{Otherwise}
\end{cases}
\end{align*}
where $\eta^{(2)} = S^{(2)}$ and $\xi^{(2)} = R^{(2)}$.
Assume the claim holds up to $l-1$, and we will show it also holds
for $l$:
\begin{align*}
O_{k,u+1,v+1}^{(l)} & =\prod_{j,i=1}^{R^{(l)}}\left(b_{ji}^{(l,k)}+\sum_{d=1}^{D^{(l-1)}}w_{dji}^{(l,k)}O_{d,u S^{(l)} + j, v S^{(l)} + i}^{l-1}\right)\\
 & =\begin{cases}
O_{k, u S^{(l)} + 1, v S^{(l)} + 1}^{(l-1)} & k\leq D\\
O_{k, u S^{(l)} + 1,v S^{(l)} + R^{(l)}}^{(l-1)} & D<k\leq2D\\
1 & \textrm{Otherwise}
\end{cases}\\
\textrm{Induction Hypothesis}\Rightarrow & =\begin{cases}
O_{k, \left(u S^{(l)}\right) \eta^{(l-1)} + 1, \left(v S^{(l)}\right) \eta^{(l-1)} + 1}^{(1)}  & k\leq D\\
O_{k, \left(u S^{(l)}\right) \eta^{(l-1)} + 1, \left(v S^{(l)} + R^{(l)} - 1\right) \eta^{(l-1)} + \xi^{(l-1)}}^{(1)} & D<k\leq2D \\
1 & \textrm{Otherwise}
\end{cases} \\
 & =\begin{cases}
O_{k, u \eta^{(l)} + 1, v \eta^{(l)} + 1}^{1}  & k\leq D\\
O_{k, u \eta^{(l)} + 1, v \eta^{(l)} + \xi^{(l)}}^{1} & D<k\leq2D \\
1 & \textrm{Otherwise}
\end{cases}
\end{align*}
Were we used the fact that $\eta^{(l)} = S^{(l)} \eta^{(l-1)}$ and $\xi^{(l)} = R^{(l)}\eta^{(l-1)} + \xi^{(l-1)} - \eta^{(l-1)}$.

Finally, we show that $O_{k,u+1,v+1}^{(L)}$ for $k\leq D$ and $0\leq u,v < H^{(L)}$ equals to the output of the single
GC layer specified in the claim:
\begin{align*}
O_{k,u+1,v+1}^{(L)} & =\prod_{j,i=1}^{R^{(L)}}\left(b_{ji}^{(L,k)}+\sum_{d=1}^{D^{(L-1)}}w_{dji}^{(L,k)}O_{d,u S^{(L)} + j, v S^{(L)} + i}^{(L-1)}\right)\\\\
 & =O_{k, u S^{(L)} + 1, v S^{(L)} + 1}^{(L-1)} \cdot O_{k,u S^{(L)} + 1,v S^{(L)}+R^{(L)}}^{(L-1)}\\
 & =O_{k, u \eta^{(L)} + 1, v \eta^{(L)} + 1}^{(1)} \cdot O_{k,u \eta^{(L)} + 1,v\eta^{(L)} + \xi^{(L)}}^{(1)}\\
 & =\left(\beta_k + \sum_{d=1}^{D^{(0)}} A_{d,k} O^{(0)}_{d, u \eta^{(L)} S^{(1)} + 1, v \eta^{(L)} S^{(1)} + 1}\right) \\
 & \phantom{=} \cdot \left(\beta_k + \sum_{d=1}^{D^{(0)}} A_{d,k} O^{(0)}_{d, u \eta^{(L)} S^{(1)} + 1, \left(v\eta^{(L)} + \xi^{(L)} - 1\right) S^{(1)} + R^{(1)})}\right)\\
 & =\left(\beta_k + \sum_{d=1}^{D^{(0)}}A_{d,k} O^{(0)}_{d,u \totstr^{(L)} + 1, v \totstr^{(L)} + 1}\right)
       \left(\beta_k + \sum_{d=1}^{D^{(0)}}A_{d,k} O^{(0)}_{d,u \totstr^{(L)} + 1,v \totstr^{(L)} + \totrec^{(L)}}\right)
\end{align*}
which is indeed equal to the single GC layer. For $k > D$, both the bias and the weights for the last layer are zero,
and thus $O^{(L)}_{k,u+1,v+1} = 1$.
\end{proof}

Finally, with the above two claims set in place, we can prove our main theorem:
\begin{proof}(\textbf{of theorem \ref{thm:main_overlaps}})
Using claim \ref{claim:many_overlaps} we can realize the networks
from claim \ref{claim:one_overlap}, for which the matricization rank for either
partition equals to:
\begin{align*}
D^{\left\lfloor \frac{H - \totrec^{(K,\nicefracfloor{H}{2})}}{\totstr} + 1 \right\rfloor \cdot \fracceil{H}{\totstr}}
\end{align*}
Since for any matricization $\left[\A\left(\Psi\right)\right]_{P,Q}$
the entries of the matricization are polynomial functions with respect to the
parameters of the network, then, according to lemma \ref{lemma:rank_everywhere},
the set of parameters of $\Phi$, that does not attain the above rank, has zero measure.
Since the union of zero measured sets is also of measure zero, then all parameters
except a set of zero measure attain this matricization rank for both partitions at once, concluding the proof.
\end{proof}

\subsection{Proof of Proposition \ref{prop:common_case}} \label{app:proofs:common_case}

Following theorem~\ref{thm:main_overlaps}, to compute the lower bound for the
network described in proposition~\ref{prop:common_case}, we need to find the first
layer for which its total receptive field is greater than $\nicefrac{H}{2}$, and then
estimate its total stride and its $\alpha$-minimal total receptive field, for $\alpha=\nicefracfloor{H}{2}$.
In the following claims we analyze the above properties of the given network:

\begin{claim}\label{claim:prop2:total_receptive}
The total stride and total receptive field of the $l$-th $B \times B$ layer in the given network, i.e. the $(2l-1)$-th GC
layer after the representation layer, are given by the following equations:
\begin{align*}
\totstr^{(2l-1)}(S^{(1)},\ldots,S^{(2l-1)}) &= 2^{l-1} \\
\totrec^{(2l-1)}(R^{(1)},S^{(1)},\ldots,R^{(2l-1)},S^{(2l-1)}) &= (2B - 1) 2^{l-1} - B + 1
\end{align*}
\end{claim}
\begin{proof}
From eq.~\ref{eq:total_stride} it immediately follows that $\totstr^{(2l-1)}(S^{(1)},\ldots,S^{(2l-1)}) = 2^{l-1}$. From eq.~\ref{eq:total_receptive},
the $2 \times 2$ stride $2$ layers do not contribute to the receptive field as $R^{(2l)} - S^{(2l)} = 0$,
which results in the following equation:
\begin{align*}
\totrec^{(2l-1)}(R^{(1)},S^{(1)},\ldots,R^{(2l-1)},S^{(2l-1)}) &=B \cdot 2^{l-1} + \sum_{i=1}^{l-1} (B-1) 2^{i-1}\\
&= B \cdot 2^{l-1} + (B-1) (2^{l-1}-1)\\
&= (2B -1) 2^{l-1} - B + 1
\end{align*}
\end{proof}

\begin{claim}
The $\alpha$-minimal total receptive field for the $l$-th $B \times B$ layer in the given network,
for $\alpha\in\N$ and $2^{l-1} \leq \alpha < 2^l-1$, always equals $(\alpha+1)$.
\end{claim}
\begin{proof}
From eq.~\ref{eq:minimal_receptive}, the following holds:
\begin{align*}
\totrec^{(2l-1,\alpha)} &= \argmin_{{\substack{\forall i\in[l], 1 \leq t_{2i-1} \leq B \\ \totrec^{(2l-1)}(t_1,1,2,2,t_3,1,\ldots,t_{2l-1},1) > \alpha}}} \totrec^{(2l-1)}(t_1,1,2,2,t_3,1,\ldots,t_{2l-1},1) \\
&= \argmin_{{\substack{\forall i\in[l], 1 \leq t_{2i-1} \leq B \\  t_{2l-1} \cdot 2^{l-1} + \sum_{i=1}^{l-1} (t_{2i-1}-1) 2^{i-1} > \alpha}}} t_{2l-1} \cdot 2^{l-1} + \sum_{i=1}^{l-1} (t_{2i-1}-1) 2^{i-1}
\end{align*}
Notice that the right term in the equation resembles a binary representation. If we limit
$t_{2i-1}$ to the set $\{1,2\}$, this term can represent any number in the set $\{0,\ldots,2^{l-1}-1\}$,
and by choosing $t_{2l-1} = 1$, the complete term can represent any number in the set $\{2^{l-1},\ldots,2^{l}-1\}$,
and specifically, for $2^{l-1} \leq \alpha < 2^l-1$, there exists an assignment for $t_{2i-1}\in\{1,2\}$ for $i\in[l-1]$
such that this terms equal $(\alpha+1)$, and thus $\totrec^{(2l-1,\alpha)}=\alpha+1$.
\end{proof}
With the above general properties for the given network, we can simplify the expression
for the lower bound given in theorem~\ref{thm:main_overlaps}:
\begin{claim}\label{claim:prop2:general_bound}
If the $l$-th $B \times B$ layer in the given network satisfies $\totrec^{(2l-1)}(R^{(1)},S^{(1)},\ldots,R^{(2l-1)},S^{(2l-1)})~>~\nicefrac{H}{2}$,
then the lower bound given in theorem~\ref{thm:main_overlaps} equals to $M^{2^{2L-2l+1}}$
\end{claim}
\begin{proof}
From the description of the network and the previous claims it holds that $D~=~M$, $H~=~2^L$,
$\totstr^{(2l-1)}~=~2^{l-1}$, and $\totrec^{(2l-1,\nicefracfloor{H}{2})} = 2^{L-1} + 1$. Substituting all
the above in eq.~\ref{eq:lower_bound} results in:
\begin{align*}
(\text{Eq.~\ref{eq:lower_bound}}) &=M^{\left\lfloor \frac{2^L - 2^{L-1} - 1}{2^{l-1}} + 1 \right\rfloor \cdot \fracceil{2^L}{2^{l-1}}}\\
&= M^{\left\lfloor 2^{L-l}+ 1 - \frac{1}{2^{l-1}} \right\rfloor \cdot 2^{L-l+1}} \\
&= M^{2^{2L-2l+1}}
\end{align*}
\end{proof}
With all of the above claims in place, we our ready to prove proposition~\ref{prop:common_case}:
\begin{proof}(\textbf{of proposition~\ref{prop:common_case}})
From claim~\ref{claim:prop2:total_receptive}, we can infer
which is the first $B \times B$ layer such that its receptive field is greater than $\nicefrac{H}{2}$:
\begin{align*}
(2B-1)\cdot2^{l-1} - B + 1 &> 2^{L-1} \\
\Rightarrow l&> \log_2 \frac{2^{L} + 2B - 2}{2B-1}\\
\Rightarrow l &= 1+ \left\lfloor \log_2 \frac{2^{L} + 2B - 2}{2B-1} \right\rfloor
\end{align*}
Combining the above with claim~\ref{claim:prop2:general_bound}, results in:
\begin{align*}
M^{2^{2L-2l+1}} &= M^{2^{2L-1- 2\left\lfloor \log_2 \frac{2^{L} + 2B - 2}{2B-1} \right\rfloor}}\\
&\geq M^{2^{2L-1- 2\log_2 \frac{2^{L} + 2B - 2}{2B-1}}} \\
(2^L \equiv H) \Rightarrow &= M^{\frac{H^2}{2} \left(1+ \frac{H -1}{2B-1}\right)^{-2}} \\
&= M^{\frac{(2B-1)^2}{2} \cdot \left(1+\frac{2B-2}{H}\right)^{-2}} 
\end{align*}
The limits and the special case for $B \leq \frac{H}{5} +1$ are both direct corollaries of the above expressions.
\end{proof}

\subsection{Proof of theorem \ref{thm:overlapping_overload}} \label{app:proofs:overlapping_overload}

We begin by proving an analogue of claim~\ref{claim:one_overlap}, where we show that for any given matricization of the grid tensor $\A(\h)$, induced
by the overlapping network realizing the function $\h$, the matricization rank is exponential. The motivation behind the construction, for when
parameters are ``unshared'', is that we can utilized the fact that there are separate sets of kernels, with local receptive fields the size of the
input, for each spatial location. Thus, each kernel can ``connect'' the index (of the grid tensor) matching its spatial location, with almost any
other index, and specifically such that the two indices come from different sets of the matricization $I\cupdot J$ of $\A(\h)$. For the ``shared''
case, we simply use polynomially more output channels to simulate the ``unshared'' case.
\begin{claim}\label{claim:overlapping_overload:existence}
For an arbitrary even partition $(I,J)$ of $\{(1,1),\ldots,(H,H)\}$, such that $|I|=|J|=\frac{H^2}{2}$,
there exists an assignment to the parameters of the network given in theorem~\ref{thm:overlapping_overload},
for either the ``unshared'' or ``shared'' settings, such that $\rank{\mat{\A(\h)}} = M^\frac{H^2}{2}$.
\end{claim}
\begin{proof}
Let $(I,J)$ be an arbitrary even partition of $\{(1,1),\ldots,(H,H)\}$, such that $|I|=|J|=\frac{H^2}{2}$,
$I = \{i_1,\ldots,i_{|I|}\}$, and $J = \{j_1,\ldots,j_{|J|}\}$, where for $k \in [\frac{H^2}{2}]$ it holds
that $i_k < i_{k+1}$ and $j_k < j_{k+1}$ (using lexical ordering), and we assume w.l.o.g. that
$i_1 = (1,1)$. We define the set $\{(q_k,p_k)\}_{k=1}^\frac{H^2}{2}$ such that $q_k = i_k$ and $p_k = j_k$
if $i_k < j_k$, and otherwise $q_k = j_k$ and $p_k = i_k$.

We prove the t``unshared'' case first, where the parameters of the first GC layers are
given by \\$\{(\w^{(c,u,v)},\bb^{(c,u,v)})\}_{c=1,u=1,v=1}^{D,H,H}$, for which
we choose the following assignment:
\begin{align*}
w^{(c,u,v)}_{m,j,i} &= \begin{cases}
(F^{-1})_{m,c} & c \leq M \tand \exists k \in \left[\frac{H^2}{2}\right], q_k = (u,v) \tand (j,i) \in \{(1,1), p_k- q_k + (1,1)\} \\
0 & \text{Otherwise}
\end{cases} \\
b^{(c,u,v)}_{j,i} &= \begin{cases}
0 & c \leq M \tand \exists k \in \left[\frac{H^2}{2}\right], q_k = (u,v) \tand (j,i) \in \{(1,1), p_k- q_k + (1,1)\} \\
1_{[c \leq M]} & \text{Otherwise}
\end{cases}
\end{align*}
where for $u,v\in[H]$ such that $q_k = (u,v)$ it holds that $p_k - q_k + (1,1) \in \{(1,1),\ldots,(H,H)\}$
because $q_k \leq p_k$. Similar to the proof of claim~\ref{claim:one_overlap},
we set all layers following the first GC layer such that the following equality holds:
\begin{align*}
\left(\A(\h)\right)_{d_{(1,1)},\ldots,d_{(H,H)}} &= \prod_{u,s=1}^{H}\sum_{c=1}^D\prod_{j,i=1}^H
\left(b_{ji}^{(c,u,v)}+\sum_{m=1}^M w_{mji}^{(c,u,v)}O_{m, u + j - 1, v + i - 1}\right)
\end{align*}
Which under the assignment to parameters we chose earlier, it results in:
\begin{align*}
\left(\A(\h)\right)_{d_{(1,1)},\ldots,d_{(H,H)}} &= \prod_{k=1}^{\nicefrac{H^2}{2}} \sum_{c=1}^{D}
\left( \sum_{m=1}^M (F^{-1})_{m,c} F_{d_{q_k},m}  \right) \left( \sum_{m=1}^M (F^{-1})_{m,c} F_{d_{p_k},m}  \right) \\
&=  \prod_{k=1}^{\nicefrac{H^2}{2}} \sum_{c=1}^{D} (F \cdot F^{-1})_{d_{q_k},c} \cdot (F \cdot F^{-1})_{d_{p_k},c} \\
&=  \prod_{k=1}^{\nicefrac{H^2}{2}} \sum_{c=1}^{D} 1_{[d_{q_k} = c]} \cdot 1_{[d_{p_k} = c]} \\
&=  \prod_{k=1}^{\nicefrac{H^2}{2}} 1_{[d_{q_k} = d_{p_k}]}
\end{align*}
which means $\mat{\A(\h)}_{(I,J)}$ equals to the Kronecker product of $\frac{H^2}{2}$ $M\times M$-identity matrices,
up to permutations of its rows and columns which do not affect its matrix rank. Thus, $\rank{\mat{\A(\h)}} = M^\frac{H^2}{2}$.

For the ``shared'' setting, we denote the parameters of the first GC layer by $\{(\w^{(d)},\bb^{(d)})\}_{c=1}^{D}$,
and set them as:
\begin{align*}
w^{(d)}_{m,j,i} &= \begin{cases}
(F^{-1})_{m,c} & \left(\exists c \in [M] \exists u,v \in [H], d = c H^2 + u H + v\right) \\
& \tand \left(\exists k \in \left[\frac{H^2}{2}\right], q_k = (u,v) \tand (j,i) \in \{(1,1), p_k- q_k + (1,1)\}\right) \\
0 & \text{Otherwise}
\end{cases} \\
b^{(d)}_{j,i} &= \begin{cases}
0 & \left(\exists c \in [M] \exists u,v \in [H], d = c H^2 + u H + v\right) \\
& \tand \left(\exists k \in \left[\frac{H^2}{2}\right], q_k = (u,v) \tand (j,i) \in \{(1,1), p_k- q_k + (1,1)\}\right) \\
1_{[d \leq M H^2]} & \text{Otherwise}
\end{cases}
\end{align*}
the parameters of the other layers are set as in the ``unshared'' case, and the proof follows similarly.
\end{proof}

In the preceding claim, we have found a separate example for each matricization, such that the matricization rank is
exponential. In the following proof of the theorem~\ref{thm:overlapping_overload}, we leverage basic properties from
measure theory to show that almost everywhere the induced grid tensor has an exponential matricization rank, under
every possible even matricization~--~without explicitly constructing such an example.
\begin{proof}(\textbf{of theorem~\ref{thm:overlapping_overload}})
For any even partition $(I,J)$ of $\{(1,1),\ldots,(H,H)\}$, according to claim~\ref{claim:overlapping_overload:existence}
there exist parameters for which $\rank{\mat{\A(\h)}}_{(I,J)} = M^\frac{H^2}{2}$, and thus according to
lemma~\ref{lemma:rank_everywhere} the set of parameters for which $\rank{\mat{\A(\h)}} < M^\frac{H^2}{2}$
is of measure zero. Since the finite union of sets of measure zero is also of measure zero, then almost
everywhere (with respect to the Lebesgue measure) the parameters results in networks such that for all even partitions $\rank{\mat{\A(\h)}}_{(I,J)} = M^\frac{H^2}{2}$.
\end{proof}

\end{document}